\documentclass[letterpaper,twocolumn,11pt]{article} 

\usepackage[margin=0.7in]{geometry}
\usepackage{times}
\usepackage[font=small,labelfont=bf]{caption}
\usepackage{titlesec}
\titlespacing\section{0pt}{8pt plus 2pt minus 2pt}{4pt plus 2pt minus 2pt}
\titlespacing\subsection{0pt}{8pt plus 2pt minus 2pt}{4pt plus 2pt minus 2pt}
\titlespacing\subsubsection{0pt}{8pt plus 2pt minus 2pt}{4pt plus 2pt minus 2pt}
\usepackage[sort,numbers]{natbib}
\usepackage{etoolbox}
\newtoggle{techrep}
\toggletrue{techrep}

\usepackage{graphicx}
\usepackage{color}
\usepackage{amsmath}
\usepackage{amssymb}
\usepackage{mathabx}
\usepackage[english]{babel}
\usepackage[mathscr]{euscript}

\iftoggle{techrep}{
\usepackage[%
bookmarksopen,bookmarksnumbered,%
hyperfigures=true,%
backref=page,%
pagebackref=true,%
breaklinks=true,%
colorlinks=true,%
citecolor=red,%
linkcolor=blue,%
urlcolor=red
]{hyperref}
\pdfinfo{
  /Author (Avik De)
}
\hypersetup{pdfpagemode=UseNone}
}{
\usepackage[hidelinks]{hyperref}
\hypersetup{pdfauthor={Avik De}}
}

\usepackage{fixme}
\fxsetup{
  status=draft,
  author=,
  layout=inline, % also try footnote or pdfnote
  theme=color
}

\usepackage{amsthm}
\theoremstyle{plain}
\newtheorem{prop}{Proposition}
\newtheorem{lemma}[prop]{Lemma}

\newtheorem{assumption}{Assumption}

\usepackage{booktabs}

\usepackage{paralist}

\graphicspath{ {../fig/} }
\usepackage{calc}
\usepackage{import}
\usepackage{transparent}
\usepackage[table,xcdraw,dvipsnames]{xcolor}
\usepackage{stmaryrd}
\usepackage{enumerate}

\makeatother

\newcommand{\mat}[1]{\left[\begin{smallmatrix}#1\end{smallmatrix}\right]}
\newcommand{\restr}[1]{\vert_{#1}}
\newcommand{\tr}{\mathrm{tr}\,}
\iftoggle{techrep} {
  
}{
  
}

\newcommand{\sB}{\mathscr{B}}
\newcommand{\sD}{\mathscr{D}}

\newcommand{\sS}{\mathscr{S}}
\newcommand{\sI}{\mathscr{I}}
\newcommand{\sA}{\mathscr{A}}
\newcommand{\sT}{\mathscr{T}}
\newcommand{\sQ}{\mathscr{Q}}
\newcommand{\sY}{\mathscr{Y}}
\newcommand{\sU}{\mathscr{U}}
\newcommand{\sV}{\mathscr{V}}
\newcommand{\D}{D}
\newcommand{\vg}{\kappa}
\newcommand{\eps}{\varepsilon}
\newcommand{\sotwo}{\mathrm{SO}(2)}
\newcommand{\setwo}{\mathrm{SE}(2)}
\newcommand{\bbR}{\mathbb{R}}
\newcommand{\bbZ}{\mathbb{Z}}
\newcommand{\R}{\mathrm{R}}
\newcommand{\disjunion}{\sqcup} % also amalg, coprod
\newcommand{\Cart}{\mathrm{Cart}}
\newcommand{\Pol}{\mathrm{Pol}}

\newcommand{\tpl}[1]{\mathsf{#1}}
\newcommand{\hyb}[2]{#1^{\tpl{#2}}}

\newcommand{\F}[2]{F^\mathrm{#2}_#1}
\newcommand{\Fr}[1]{{F^\mathrm{#1}}}
\newcommand{\J}{J}
\newcommand{\hw}{h_\mathrm{w}}
\newcommand{\Fvg}{\hyb{F}{v}}
\newcommand{\gvg}{h_\vg}
\newcommand{\Fbhiso}{\hyb{F}{fa}}
\newcommand{\Fw}{\widetilde{\hyb{F}{s}}}

\newcommand{\epsvh}{\eps_\mathrm{r}}
\newcommand{\delmax}{\bar\delta_\mathrm{max}}
\newcommand{\sma}{{\bf a}}
\newcommand{\smA}{{\bf A}}
\newcommand{\smq}{{\bf q}}
\newcommand{\smM}{{\bf M}}
\newcommand{\smC}{{\bf C}}
\newcommand{\smN}{{\bf N}}
\newcommand{\smy}{{\bf y}}
\newcommand{\smh}{{\bf h}}
\newcommand{\smU}{\Upsilon}
\newcommand{\smH}{{\bf H}}
\newcommand{\smg}{{\bf g}}

\newcommand{\tg}{k_t}

\iftoggle{techrep}{
\title{\LARGE
The Penn Jerboa: A Platform for Exploring Parallel Composition of Templates \\ {\large Technical Report to Accompany~\cite{de_parallel_2015}}}
}{
\title{\LARGE \bf
Parallel Composition of Templates for Tail-Energized Planar Hopping}
}

\author{Avik De$^\star$ \iftoggle{techrep}{and}{\and} Daniel E.\ Koditschek$^\star$% <-this % stops a space
  \thanks{$^\star$Electrical and Systems Engineering, University of Pennsylvania, Philadelphia, PA, USA.
  {\tt\footnotesize \{avik,kod\}@seas.upenn.edu}.}
  \thanks{This work was supported in part by the ARL/GDRS RCTA project, Coop.\ Agreement \#W911NF-10–2−0016 and in part by NSF grant \#1028237.}
}

\begin{document}

\maketitle

% \iftoggle{techrep}{\small\tableofcontents\normalsize}{}

\thispagestyle{empty}
\pagestyle{empty}

%%%%%%%%%%%%%%%%%%%%%%%%%%%%%%%%%%%%%%%%%%%%%%%%%%%%%%%%%%%%%%%%%%%%%%%%%%%%%%%%

\begin{abstract}
\iftoggle{techrep}{
We have built a 12DOF,  passive-compliant legged,  tailed biped actuated by four  brushless DC motors. We anticipate that this machine will  achieve varied  modes of quasistatic and dynamic balance, enabling
 a broad range of locomotion tasks  including sitting, standing, walking, hopping, running, turning, leaping, and more. 
Achieving this diversity of behavior with a single under-actuated body, requires a correspondingly diverse array of controllers, motivating our interest in compositional techniques that promote mixing and reuse of a relatively few base constituents to achieve a combinatorially growing array of available choices. Here we report on the development of one important example of such a behavioral programming method, the construction of a novel monopedal sagittal plane hopping gait through parallel composition of four decoupled 1DOF base controllers. 

For this example behavior, the legs are locked in phase and the body is fastened to a boom  to restrict motion to the sagittal plane. The platform's locomotion is powered by the hip motor that adjusts leg touchdown
angle in flight and balance in stance, along with a tail motor that adjusts body shape in flight and drives energy
into the passive leg shank spring during stance. The motor control signals arise from the application in parallel
of four simple, completely decoupled 1DOF feedback laws that provably stabilize in isolation four corresponding
1DOF abstract reference plants. Each of these abstract 1DOF closed loop dynamics represents some simple but
crucial specific component of the locomotion task at hand. We present a partial proof of correctness for this parallel
composition of “template” reference systems along with data from the physical platform suggesting these templates
are “anchored” as evidenced by the correspondence of their characteristic motions with a suitably transformed image
of traces from the physical platform.
}{
We have built a 4DOF tailed monoped that hops along a boom permitting free sagittal plane motion. This underactuated platform is powered by a hip motor that adjusts leg touchdown angle in flight and balance in stance, along with a tail motor that adjusts body shape in flight and drives energy into the passive leg shank spring during stance. The motor control signals arise from the application in parallel of four simple, completely decoupled  1DOF feedback laws that provably stabilize in isolation four corresponding 1DOF abstract reference  plants.  Each of these abstract 1DOF closed loop dynamics represents some simple but crucial specific component of the locomotion task at hand. We present a partial proof of correctness for this parallel composition of ``template'' reference systems along with data from the physical platform suggesting these templates are ``anchored'' as evidenced by the correspondence of their characteristic motions with a suitably transformed image of traces from the physical platform.
}
\end{abstract}

%%%%%%%%%%%%%%%%%%%%%%%%%%%%%%%%%%%%%%%%%%%%%%%%%%%%%%%%%%%%%%%%%%%%%%%%%%%%%%%%

\section{Introduction}

The control of power-autonomous, dynamic legged robots that have a high number of degrees of freedom (DOF) is made difficult by a number of factors including
\begin{inparaenum}[(a)]
\item under-actuation necessitated by power-density constraints, 
\item the existence of significant inertial coupling and Coriolis forces that are hard or impossible to cancel,
\item variable ground affordance, 
\item often hard-to-measure and necessarily rapid hybrid transitions.
\end{inparaenum}
In the face of these challenges, some popular methods of controller design, such as hybrid zero dynamics~\cite{westervelt_feedback_2007}---which are ``exact'' in their domain of applicability but require extremely accurate qualitative and quantitative models---may be challenging to implement in unstructured environments or on imperfectly characterized machines. Similarly, methods depending on local linearizations of the typically (highly) nonlinear dynamics found in dynamically dexterous locomotion and manipulation systems \cite{tedrake_underactuated_2009,holmes_dynamics_2006} typically suffer from small basins of attraction \cite{Buehler_Koditschek_Kindlmann_1989} and (to our knowledge) high sensitivity to parameters.\footnote{In some robotics settings these disadvantages of the exact or local linearized control paradigm can be effectively remedied by recourse to parameter adaptation \cite{whitcomb_comparative_1993}, but in our experience, such methods are too ``laggy'' to work in this hybrid dynamics domain with its intrinsically abrupt and rapidly switching characteristics.}

%%%%%%%%%%%%%%%%%%%%%%%%%%%%%%%%%%%%%%%%%%%%%%%%%%%%%%%%%%%%%%%%%%%%%%%%%%%%%%%%

\begin{figure*}[t]
\minipage[b]{\textwidth}
\centering
\def\svgwidth{0.7\textwidth}
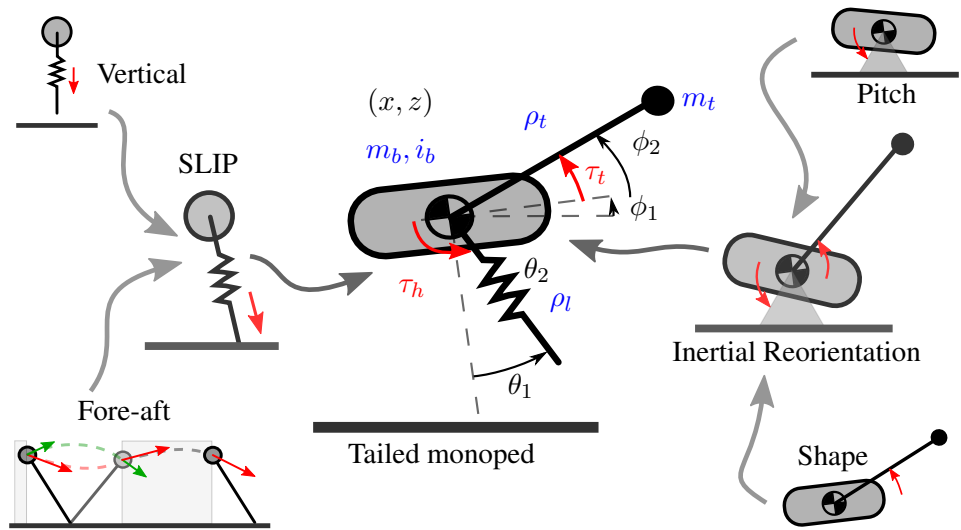
\endminipage\hfill
\caption{Control of a hopping behavior expressed as a hierarchical composition of closed-loop templates. Notionally, the grey arrows represent directed template$\rightarrow$anchor relations. {\bf Center:} A model of the tailed monoped physical platform on which we implement tail-energized planar hopping, labeled with configuration variables (black), actuators (red), and model parameters (blue).}
\label{fig:Tree}
\end{figure*}
Observation (a) suggests that modularity of operation (i.e., wherein different combinations of actuators are used to effect distinctly different dynamical goals at different stages within the task cycle) will be a hallmark of practical locomotion platforms. Observations (b) and (c) imply that simpler, less exact but potentially more robust representations of the principal dynamical effects likely to prevail across a wide range of substrates may offer a tractable means of working with  rather than fighting against, or learning exactly the highly varied dynamical details. Observation (d) implies that higher authority sensorimotor control activity ought to target continuous phases of the locomotion cycle, leaving the transition event interventions to more passive and mechanical sources of regulation \cite{Burden_Revzen_Sastry_Koditschek_2014}.  In sum, these observations motivate the search for modular,  reduced order representations of locomotion task constituents that are specialized to couple selected actuation affordances to particular DOFs at particular phases of the locomotion cycle. The value of such  component task representatives remains  hostage to the availability of methods for composing them in a stable manner. 
% These observations point towards a need for reduced-order models that are descriptive of the dynamics (e.g. many multi-legged running animals and robots reduce to SLIP-like limiting behavior \cite{holmes_dynamics_2006}). 

% %%%%%%%%%%%%%%%%%%%%%%%%%%%%%%%%%%%%%%%%%%%%%%%%%%%%%%%%%%%%%%%%%%%%%%%%%%%%%%%%

% \begin{figure}[t]
% \minipage[b]{0.2\columnwidth}
% \centering
% \includegraphics[width=1.75cm]{../fig/vidcap5s.jpg}
% \endminipage\hfill
% \minipage[b]{0.2\columnwidth}
% \centering
% \includegraphics[width=1.75cm]{../fig/vidcap4s.jpg}
% \endminipage\hfill
% \minipage[b]{0.2\columnwidth}
% \centering
% \includegraphics[width=1.75cm]{../fig/vidcap3s.jpg}
% \endminipage\hfill
% \minipage[b]{0.2\columnwidth}
% \centering
% \includegraphics[width=1.75cm]{../fig/vidcap2s.jpg}
% \endminipage\hfill
% \minipage[b]{0.2\columnwidth}
% \centering
% \includegraphics[width=1.75cm]{../fig/vidcap1s.jpg}
% \endminipage
% \caption{Snapshots from apex to apex of tail-energized planar hopping (\S\ref{sec:TailedMonoped}) implemented on a new robot platform---the Penn Jerboa (\S\ref{sec:Experiments}).}
% \label{fig:robot}
% \end{figure}

%%%%%%%%%%%%%%%%%%%%%%%%%%%%%%%%%%%%%%%%%%%%%%%%%%%%%%%%%%%%%%%%%%%%%%%%%%%%%%%%

\newlength{\rowfiga}
\newlength{\rowfigb}
\iftoggle{techrep}{
\setlength{\rowfiga}{0.2\textwidth}
\setlength{\rowfigb}{3.5cm}
}{
\setlength{\rowfiga}{0.2\textwidth}
\setlength{\rowfigb}{3.5cm}
}
\begin{figure*}[t]
\minipage[b]{\rowfiga}
\centering
\includegraphics[width=\rowfigb]{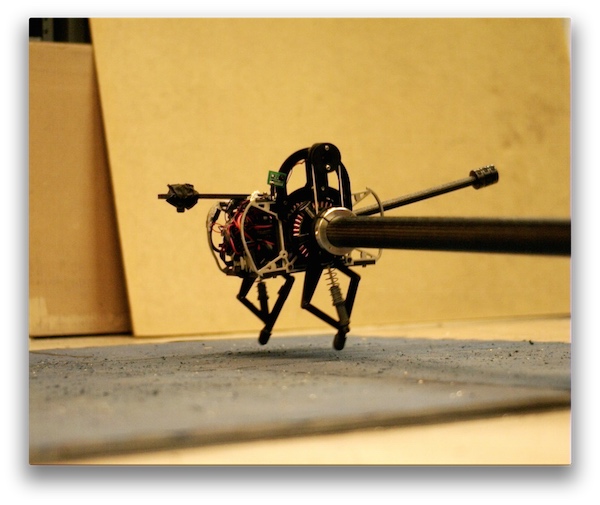}
\endminipage\hfill
\minipage[b]{\rowfiga}
\centering
\includegraphics[width=\rowfigb]{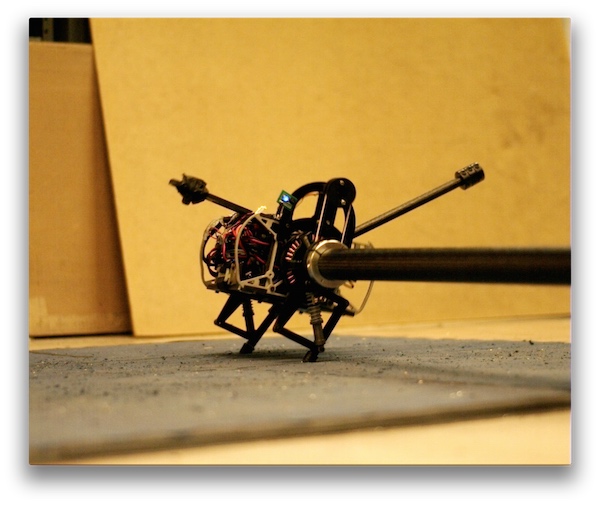}
\endminipage\hfill
\minipage[b]{\rowfiga}
\centering
\includegraphics[width=\rowfigb]{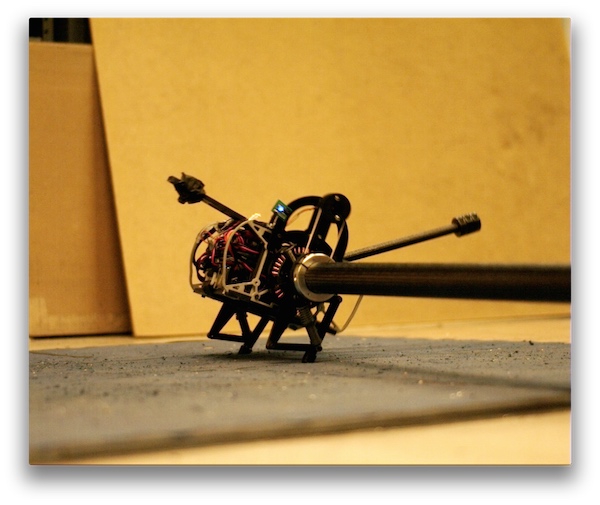}
\endminipage\hfill
\minipage[b]{\rowfiga}
\centering
\includegraphics[width=\rowfigb]{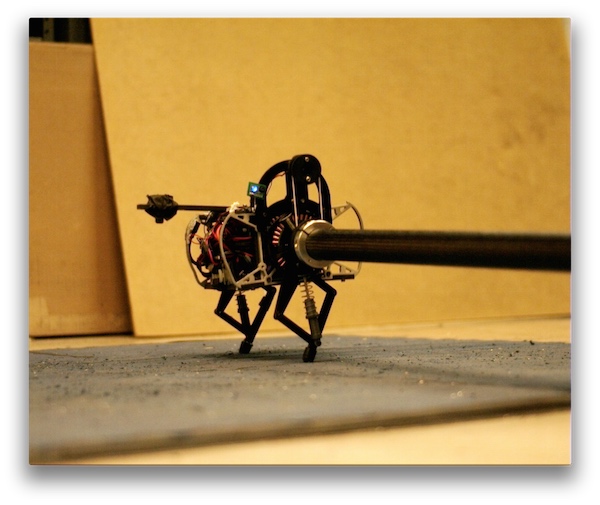}
\endminipage\hfill
\minipage[b]{\rowfiga}
\centering
\includegraphics[width=\rowfigb]{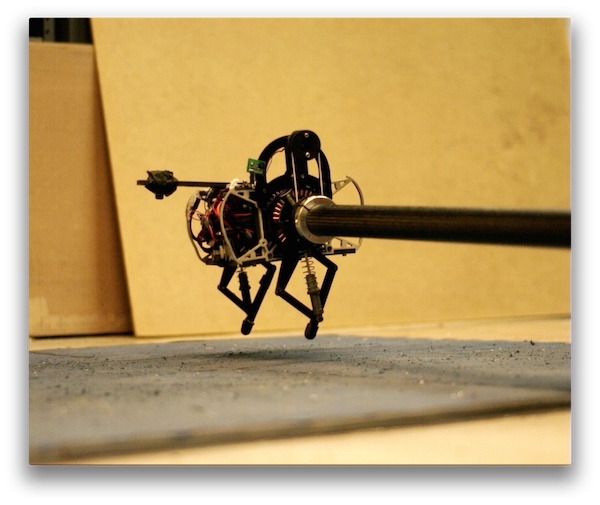}
\endminipage
\caption{Snapshots from apex to apex of tail-energized planar hopping (\S\ref{sec:TailedMonoped}) implemented on a new robot platform---the Penn Jerboa (\S\ref{sec:Experiments}).}
\label{fig:robot}
\end{figure*}

%%%%%%%%%%%%%%%%%%%%%%%%%%%%%%%%%%%%%%%%%%%%%%%%%%%%%%%%%%%%%%%%%%%%%%%%%%%%%%%%

This report introduces a novel locomotion platform, the Penn Jerboa, Fig.\ \ref{fig:robot}, to put a slowly maturing formalism for the composition of such modules to a practical test. We adopt the template-anchor\footnote{
The template-anchor relation as exemplied in various physical~\cite{Buehler_Koditschek_Kindlmann_1989,nakanishi_brachiating_2000} and numerical~\cite{Saranli_Schwind_Koditschek_1998} studies associates a pair of smooth vector fields, $\hyb{f}{T}, \hyb{f}{A}$ on a pair of smooth spaces, $\sT \subset \sA$   via the condition that $\sT$ is an attracting invariant submanifold of the anchor field, $\hyb{f}{A}$, whose restriction dynamics is conjugate to that of the template field, $\hyb{f}{T} \sim \hyb{f}{A} \restr{\sT}$ (where $\sim$ denotes equivalence up to smooth change of coordinates). In this paper, we are dealing with hybrid fields and flows for which the extended definition and its verification is a bit more intricate. Thus exceeding the scope and length constraints of the present paper,  we will treat the hybrid template-anchor relation as an intuitive notion here.} 
framework \cite{full_templates_1999}
to represent this machine's 4DOF steady sagittal plane running  as the hierarchical composition of  the low DOF constituents  depicted in Fig.\ \ref{fig:Tree}.  At the leaves of this hierarchy tree, we introduce four different 1DOF templates that emerge  from the decades old bioinspired running literature \cite{holmes_dynamics_2006,Blickhan_Full_1993}, joined by a new arrival from recent work on bioinspired tails \cite{libby_tail-assisted_2012,johnson_tail_2012}.
We apply  the four decoupled 1DOF control laws associated with these isolated ``leaf'' templates directly to the (highly dynamically coupled) physical platform and demonstrate empirically steady sagittal plane running (on a circular boom) whose body motions reveal, when viewed in the appropriate coordinates, Fig.\ \ref{fig:compositionPlots}, striking similarity to the corresponding isolated 1DOF constituents.  We show (up to a still unproven technical conjecture) that the appropriate two pairs of these four 1DOF leaf templates  are formally anchored by the two ``interior'' 2DOF templates depicted in Fig. \ref{fig:Tree}, in the sense that the 1DOF systems define attracting invariant submanifolds of the 2DOF systems that exhibit conjugate restriction dynamics. We conjecture, as well,  that the two interior nodes (the 2DOF templates) of the figure are in turn formally anchored by a physically realistic dynamical model of the closed loop Penn Jerboa in the sagittal plane. The data of Fig.~\ref{fig:compositionPlots} support this hypothesis, but we have not yet succeeded in completing the proof beyond the embedding and invariance properties.
% We conjecture that incrementally complex systems will attain cross-products of low-DOF reduced-order models (or templates \cite{full_templates_1999}). In this paper we investigate planar hopping executed on a kinematically 6DOF, dynamically 4DOF, 2DOF-actuated platform, and develop a ``modular'' means of building controllers by snapping together very simple 1DOF controllers associated with four 1DOF templates \cite{full_templates_1999} that specify in isolation the various components of the overall task.

Notwithstanding the specifics of our compositional approach to its control, we believe that the  new physical platform  is itself of independent interest by virtue of its added appendage (the ``tail''), opening up a multiplicity of diverse uses for both of its two revolute actuators. 
Note again, however, this diversity of uses cannot be achieved without some recourse to behavioral modularity.  In that light, we are particularly attracted by these simple low-DOF template controllers. In our experience, such constructions have the hope of succeeding in unstructured outdoor settings, since they
build on the relatively robust template dynamics.
% Unlike inverse dynamics (which is hopelessly sensitive to parameter variations in the absence of complex nonlinear adaptive mechanisms \cite{whitcomb_comparative_1993}, which, in turn are often too ``laggy'' to be useful for dynamic locomotion stances), we believe that these simple low-DOF controllers have the hope of succeeding in unknown outdoor settings, since they build on the relatively robust template dynamics.

\subsection{Relation to Prior Literature}

This ``compositional'' method of controller synthesis was pioneered empirically by Raibert \cite{raibert_legged_1986} for planar and 3D hopping machines, and we develop our planar hopping behavior by building up from those ideas.  Our physical platform (Fig.\ \ref{fig:Tree} center) forgoes Raibert's prismatic shank actuator, and instead places that actuator in an inertial appendage. This motivates us to explore how tails can be ``recycled'' from their transitional agility duties \cite{libby_tail-assisted_2012,johnson_tail_2012}, now repurposed to substitute for Raibert's shank actuator and play the role of steady-state running energizer in the sagittal plane. Apart from their use in transitional maneuvers (inertial control in free-falling lizards \cite{gillis_losing_2009} and robots \cite{libby_tail-assisted_2012,johnson_tail_2012} or in turning lizards \cite{higham_maneuvering_2001} and robots \cite{pullin2012dynamic}) it has recently been discovered that kanagaroos do positive work with their tails in a quasistatic pentapedal gait \cite{oconnor_kangaroos_2014}.
In our implementation,  the tail contributes the reorientation function in flight, and the energetic ``pump'' function in stance (albeit in a dynamic fashion). We are not aware of prior robotic locomotion work wherein a tail is used to help power the stance phase.

%%%%%%%%%%%%%%%%%%%%%%%%%%%%%%%%%%%%%%%%%%%%%%%%%%%%%%%%%%%%%%%%%%%%%%%%%%%%%%%%

\begin{table}[t]
  \caption{List of Symbols}
  \label{tab:symbols}
  \begin{center}
  \footnotesize
  \begin{tabular}{@{}ll@{}}
    \toprule
    $i\in\bbZ_2$ & 
    Hybrid mode, where 1 is stance, 2 is flight \\
    $\sD_i^\star$ & 
    Domain for template $\star$ in mode $i$ \\
    $f_i^\star : \sD_i^\star \rightarrow T\sD_i^\star$ &
    Vector field in mode $i$ \\
    $r_i^\star : \partial\sD_i^\star \rightarrow \sD_{i+1}^\star$ &
    Reset map from mode $i$ to $i+1$ \\
    $F_i^\star : \sD_i^\star \rightarrow \partial\sD_i^\star$ &
    Mode $i$ flow evaluated at the next transition \\
    $F^\star = F_2^\star \circ F_1^\star$ &
    Return map at touchdown (TD) event \\
    $p_i^\star(x, u)$ &
    Plant to which we apply $u = g_i(x)$ to get $f_i^\star$ \\
    $I_d \in \bbR^{d\times d}$ &
    Identity matrix of size $d$\\
    $\J = \mat{0 & -1\\1&0}$ &
    Planar skew-symmetric matrix \\
    $e_i\in\bbR^d$ &
    $i^\mathrm{th}$ standard basis vector \\
    $\R : S^1 \rightarrow \sotwo$ &
    Map from angle to rotation matrix \\
    $T x = (x, \dot x)$ &
    Tangent vector associated with $x$ \\
    $\D_x y$ &
    Jacobian matrix $\partial y_i / \partial x_j$ \\
    $\vg \in \bbR_+$ &
    SLIP radial velocity gain (\S\ref{sec:MBHop}) \\
    $\gvg \in \bbR \rightarrow \bbR_+$ &
    Map from radial TD velocity to $\vg$ (\S\ref{sec:SpringEnergization}) \\
    $\gamma : \bbR \rightarrow S^1$ &
    Fore-aft model stance sweep angle (\S\ref{sec:MBHop}) \\
    $\beta : \bbR \rightarrow S^1$ &
    Raibert touchdown angle function (\ref{eq:RaibertController}) \\
    $\hw : \bbR^2 \rightarrow \bbR^2$ &
    Cartesian to Polar TD velocity (\S\ref{sec:AnchoringTemplates}) \\
    % $\smq, \sma_i, \smA_i, \smM, \smC, \smN, \smU$ &
    % Lagrangian mechanics terms \cite{johnson_legged_2013} \\
    \bottomrule
  \end{tabular}
  \normalsize
  \end{center}
\end{table}

%%%%%%%%%%%%%%%%%%%%%%%%%%%%%%%%%%%%%%%%%%%%%%%%%%%%%%%%%%%%%%%%%%%%%%%%%%%%%%%%

\begin{table*}[t]
  \caption{Template Controllers}
  \label{tab:controllers}
  \begin{center}
  \begin{tabular}{@{}lll@{}}
    \toprule
    Tail energy pump &
    $\hyb{g}{v}_1(x) = k_t \cos(\angle x)$ &
    (\ref{eq:TailControl}) \\[2pt]
    Raibert stepping \cite{raibert_legged_1986} &
    $\hyb{g}{fa}_2(\dot x) = \beta^*(\dot x) + k_p (\dot x - \dot x^*)$ &
    (\ref{eq:RaibertController}) \\[2pt]
    Raibert pitch correction \cite{raibert_legged_1986} &
    $\hyb{g}{p}_1(a_1,\dot a_1) = -k_g k a_1 - k_g \dot a_1$ &
    (\ref{eq:HIRCL}) \\[2pt]
    Shape reorientation \cite{johnson_tail_2012} &
    $\hyb{g}{sh}_2(a_2,\dot a_2) = -k_g k a_2 - k_g \dot a_2$ &
    (\ref{eq:HIRCL}) \\
    \bottomrule
  \end{tabular}
  \end{center}
\end{table*}

%%%%%%%%%%%%%%%%%%%%%%%%%%%%%%%%%%%%%%%%%%%%%%%%%%%%%%%%%%%%%%%%%%%%%%%%%%%%%%%%

\subsection{Contributions of the Paper}

This paper contributes both to the theory and practice of dynamical legged locomotion. 

The principal theoretical contributions are: 
\begin{inparaenum}[(i)]
\item a new  (slightly simplified) further abstraction (\S\ref{sec:CouplingSLIP}) of the longstanding SLIP running model \cite{holmes_dynamics_2006} as a formal cross-product of previously proposed vertical \cite{secer_control_2013} and fore-aft \cite{arslan_reactive_2012} templates;
\item a stability proof (modulo a restrictive assumption~\ref{def:PendularAssumptions}) of the parallel composition\footnote{By this term we mean the application to the (coupled) plant $\hyb{p}{s}(x,u)$ (\S\ref{sec:CouplingSLIP}) of a decoupled control law, $u=\hyb{g}{v}(x_1) \times \hyb{g}{fa}(x_2)$, taken directly from~\eqref{eq:TailControl},~\eqref{eq:RaibertController}, respectively.}
of Raibert's \cite{raibert_legged_1986} stepping controller \eqref{eq:RaibertController} with our new energy pump \eqref{eq:TailControl} in Proposition~\ref{prop:SLIPComposition}; and
\item a proof of local stability in the inertial reorientation model \eqref{eq:HIRbody} of the parallel composition \eqref{eq:HIRCL} of Raibert's \cite{raibert_legged_1986} pitch stabilizer and  the tail reorientation controller \cite{johnson_tail_2012} in Proposition~\ref{prop:HIRStability}.
\end{inparaenum}

The empirical contributions of the paper are:
\begin{inparaenum}[(i)]
\item design and implementation of a working tailed biped platform, the Penn Jerboa\iftoggle{techrep}{}{~\cite{CompositionHoppingTR}} (Fig. \ref{fig:robot}); 
\item physical demonstration of the (provably correct--Proposition~\ref{prop:NegativeDamping}) oscillatory spring-energization scheme for vertical hopping; and 
\item experimental evidence supporting the hypothesis that our final  parallel composition   of the  four isolated controllers does indeed anchor the corresponding templates in the Jerboa body (Fig. \ref{fig:compositionPlots}).
\end{inparaenum}

% The novel contributions of this paper are:
% \begin{inparaenum}[(i)]
% \item an oscillatory spring-energization scheme for vertical hopping,
% \item a presentation of SLIP as a cross-product of independent vertical and fore-aft templates, 
% \item a 2DOF hybrid attitude stabilization model building on 1DOF inertial reorientation \cite{johnson_tail_2012}, 
% \item the development of ``cross-product'' hybrid dynamical systems together with the consequent clock-synchronization and return map analysis,
% \item a proof of local stability of parallel composition on a planar tailed monoped anchor (including a proof of stability of a full 3DOF model of the Raibert point-mass planar hopper \cite{raibert_legged_1986} built upon a simple but descriptive fore-aft model), and
% \item empirical demonstration of this result.
% \end{inparaenum}

\iftoggle{techrep}{
While the idea of parallel composition is appealing, the difficulty of such a composition arises from the natural transfer of energy between different compartments \cite{Eriksson_1971}\footnote{We use this term here to stand for subsystems (here, disjoint subsets of the physical degrees of freedom) that exchange a resource (here, energy).} in
a mechanical system operating in a dynamical regime. 
In our setting, some degree of coupling across compartments is crucial to the underlying design concept of driving the leg spring through torques generated ``far away'' in the tail. Thus, a naive approach of looking for exactly decoupled body
dynamics is not fruitful\footnote{For instance, for hopping with the tailed monoped, the tail actuator and hip actuator seemingly work on differently ``binned'' tail and leg DOFs, but we energize the robot body with the tail through the leg spring.}. Instead, we analyze stability
properties of (hybrid) closed-loop templates—--which are not
specifically associated to any body--—without paying attention
to the input structure. In agreement with intuition, we
find (\S\ref{sec:PhysicalDecoupling}) that minimization of cross-template transfer of
energy--—through either the flows or the reset maps—--results
in a successful composition.
}{}
% \fxerror{
% Avik, I think  we would do well to introduce as well a uniform symbol name for  the various "plants", e.g., p^*_i(x,u) relative to which we will introduce the control, u=g_i(x) ) to get the closed loop vector fields, f^*_i, where * once again ranges over the label hierarchy  }
% }

\section{Preliminaries: Organization and Notation}
\label{sec:Prelim}

% I'm pretty sure when we get this paper more under control that you'll want to reorganize the 1 dof templates into subsections of their major component sections so that Section II & III & IV become subsections of a new Section III (the slip template as a parallel composition, III.C,  of vertical, III.A, & fore-after III.B subtemplates), whereas V is reorganized into (much shorter) analogous subsections (IV.A = HIR subtemplate, IV.B = Shape subtemplate, IV.C = their parallel composition). Notice that this leaves a "preliminary" section II in place for introduction of terminology (we've got templates, subtemplates, bodies, and anchors), notation introducintg the table, and emphasizing the logical construction of the paper as reflecting the logical composition of subtemplates into templates and then into the physical body for anchoring. 

Table \ref{tab:symbols} contains a list of important symbols in this paper, including a set of symbols for describing hybrid dynamical systems. We adopt the modeling paradigm from Definition 1 in \cite{burden_dimension_2011}, representing a hybrid dynamical system by the tuple $(\sD, f, r)$ as defined in Table \ref{tab:symbols}. We only consider two hybrid modes in this paper: ballistic flight, and a stance phase arising from a sticking contact at the ``toe''.
% and also enforce that the guard set is strictly $\partial \sD$.

Superscripts on each of these symbols denote the \emph{hybrid template} that it is a part of, e.g. $\hyb{\star}{v}$ for controlled vertical hopping (\S\ref{sec:Radial}). 
The layout of the paper roughly reflects the template-anchor hierarchy
depicted in Fig. \ref{fig:Tree}. Namely, there are two intermediate 2DOF templates---the SLIP, $\tpl{s}$, and the inertial reorientation, $\tpl{a}$----that comprise  the tailed monoped, $\tpl{tm} = \{\tpl{s}, \tpl{a}\}$. They, in turn, are comprised of the vertical, $\tpl{v}$, and fore-aft, $\tpl{fa}$, 1DOF templates, $\tpl{s}=\{\tpl{v}, \tpl{fa}\}$, and respectively, the shape, $\tpl{sh}$, and pitch, $\tpl{p}$, 1DOF templates, $\tpl{a} = \{\tpl{sh}, \tpl{p}\}$. We endow the 1DOF templates at the lowest level with an exemplar plant, with respect to which we will develop controllers for the four template plants, in isolation.

Sections \ref{sec:SLIP}-\ref{sec:HIR} present the 2DOF $\tpl{s}, \tpl{a}$ templates that are directly anchored in the robot body (\S\ref{sec:TailedMonoped}), and within them contain descriptions of the subtemplates (e.g. \S\ref{sec:Radial}, \ref{sec:FA})---as simple exemplar 1DOF anchoring bodies and corresponding control laws---that comprise in isolation the constituent desired limiting behaviors that we seek to embody simultaneously in our physical system. 
Each of the template controllers in this suite is necessarily simple by dint of its origin as a feedback law for a highly abstract 1DOF task exemplar.  We hypothesize that this combination of algorithmic simplicity and task specialization may lend robustness in the empirical setting since control policies are not
sensitive to, and certainly avoid cancellation of, forces arising from dynamical coupling in the anchoring body.
% The suite of simple (by dint of the simplicity of their models), robust control policies are not sensitive to, and certainly avoid cancellation of, the various sources of crosstalk arising from their dynamical coupling in the anchoring body.

We emphasize that these coupling-na\"{i}ve feedback laws (summarized in Table \ref{tab:controllers}) are simply ``played back'' (modulo scaling) in the 6DOF body (\S\ref{sec:TailedMonoped}) with all its complicated true dynamical coupling. We show formally through various propositions in this paper that nevertheless the stability of the templates and subtemplates persists through composition for the distal segments of the tree (Fig.~\ref{fig:Tree})---SLIP as a composition of vertical hopping and fore-aft speed control, and attitude stabilization as a composition of inertial reorientation and Raibert's pitch control. We provide some preliminary suggestions about the composition of SLIP ($\mathrm{s}$) with attitude ($\mathrm{a}$) compartments (center of Fig.~\ref{fig:Tree}), but a full analysis is left to future work. However, we offer empirical data in \S\ref{sec:Experiments} showing how this idea has resulted in promising qualitative behavior on the Jerboa robot (Fig.~\ref{fig:compositionPlots}, video attachment).

%%%%%%%%%%%%%%%%%%%%%%%%%%%%%%%%%%%%%%%%%%%%%%%%%%%%%%%%%%%%%%%%%%%%%%%%%%%%%%%%

\iftoggle{techrep}{
\begin{table}[t]
  \caption{Physical Parameters (all scalars unless noted)}
  \label{tab:parameters}
  \begin{center}
  \small
  \begin{tabular}{@{}ll@{}}
    \toprule
    $k_t$ &
    Tail gain (\ref{eq:TailControl}) \\
    $k_p$ &
    Raibert speed controller gain (\ref{eq:RaibertController}) \\
    $k$ &
    Inertial reorientation generalized damper gains (\ref{eq:HIRCL}) \\
    $k_g$ &
    Inertial reorientation graph error gain (\ref{eq:HIRCL}) \\
    $\sigma, \omega$ &
    Dissipation, frequency of spring-damper (\S\ref{sec:Radial}) \\
    $\eps$ &
    Saturation parameter for tail controller (\ref{eq:TailControl})) \\
    $\eps_r$ &
    Stability margin for vertical hopping (Proposition \ref{prop:SLIPComposition}) \\
    $\eps_a$ &
    Arbitrarily small orientation error (Proposition \ref{prop:HIRStability}) \\
    $m_b, i_b$ &
    Mass, inertia of robot body (\S\ref{sec:TailedMonoped}) \\
    $\rho_l, \rho_t$ &
    Leg, tail link lengths (\S\ref{sec:SLIP},\ref{sec:TailedMonoped}) \\
    $k_s$ &
    Hooke's law leg spring constant (\S\ref{sec:SLIP},\ref{sec:TailedMonoped}) \\
    \bottomrule
  \end{tabular}
  \normalsize
  \end{center}
\end{table}
}{
\textbf{Note:} Due to space constraints, we have moved the proofs as well as additional experimental results  to a companion technical report~\cite{CompositionHoppingTR}.
}

%%%%%%%%%%%%%%%%%%%%%%%%%%%%%%%%%%%%%%%%%%%%%%%%%%%%%%%%%%%%%%%%%%%%%%%%%%%%%%%%

\section{The (2DOF) SLIP Template}
\label{sec:SLIP}

%%%%%%%%%%%%%%%%%%%%%%%%%%%%%%%%%%%%%%%%%%%%%%%%%%%%%%%%%%%%%%%%%%%%%%%%%%%%%%%%

\subsection{Controlled Vertical Hopping (1DOF)}
\label{sec:Radial}

For a successful hopping behavior, energy must be periodically injected into the robot body to compensate for losses. We simplify the analysis here to a 1DOF vertically-constrained point-mass which can alternate between stance phase (during which the actuator has affordance) and a ballistic (passive) flight phase. It has been shown in the past empirically \cite{raibert_legged_1986} and analytically \cite{koditschek_analysis_1991} that an impulse at the bottom of stance can produce a stable limit cycle, in the presence of a spring for energy storage. In this paper, we consider a different strategy of an actuator forcing the damped spring by applying forces in a phase-locked manner. This choice of input representative is made with an eye toward using a tail actuator exerting inertial reaction forces on the spring (this model is formally instantiated \S\ref{sec:TailedMonoped}).
Intuitively, this can be thought of as negative damping \cite{secer_control_2013} (effectively cancelling losses by physical damping).

Throughout this paper, we make the following assumption inspired by~\cite{raibert_legged_1986}:
\begin{assumption}[Stance duration]
The duration of stance, $T_s$, is approximately constant. 
\end{assumption}
This essentially asserts that the damping losses or actuator forces are relatively small compared to the spring-mass dynamics (in their effect on the liftoff condition).

We build upon the ``linear spring'' analysis in \cite{koditschek_analysis_1991} for our vertical hopping exemplar body and closed-loop template. For a spring-mass-damper system with spring deflection $\chi$, damping coefficient $\bar\beta$ and natural frequency $\omega$
\begin{align}
\label{eq:SpringMass}
% \ddot \chi + 2 \omega \bar\beta \dot \chi + \omega^2(1 + \bar\beta^2) \chi = -\tau,
\ddot \chi + 2 \omega \bar\beta \dot \chi + \omega^2 \chi = \tau.
\end{align}
% 
% We reparameterize with $\sigma := \omega \bar\beta$, 
% We make the following assumption (to be used in Proposition~\ref{prop:NegativeDamping}):
% \begin{assumption}
% \label{def:VerticalAssumptions}
% The hopper parameters ensure
% \begin{inparaenum}[(i)]
% \item $\sigma^2 + \omega^2 = 1$ (such that\footnote{There is no loss of generality since this amounts to choosing a time scale, preserving all orbits and their limiting properties.} $W := \mat{1 & \sigma \\ 0 & \omega}$ has a simple inverse $W^{-1} = \mat{1 & -\sigma/\omega \\ 0 & 1/\omega}$) 
% and
% \item $\omega > \sigma$.
% \end{inparaenum}
% $\bar\beta<1$.
% \end{assumption}
% 
% $W$ from Assumption \ref{def:VerticalAssumptions} is a transformation to real-canonical form for the unforced system, i.e. if $x := W \mat{\chi\\\dot\chi}$,
With the change of coordinates $x_1 := \chi$, $x_2 := \dot\chi/\omega$,
\begin{align}
% \dot x = \hyb{p}{v}_1(x,\tau) := (-\sigma I - \omega J) x + \mat{\sigma \\ \omega} \tau,
\dot x = \hyb{p}{v}_1(x,\tau) := -\omega J x + e_2^T(-2 \bar\beta \omega x_2 + \tau/\omega),
\end{align}
and the hybrid reset events occur at $x_1 = 0$ (corresponding physically to the touchdown and liftoff events at $\chi =0$).

\subsubsection{Oscillatory Spring Energization}
\label{sec:SpringEnergization}

We choose the physically motivated control strategy
\begin{align}
\label{eq:TailControl}
\tau := \tfrac{\tg x_2}{\Vert x \Vert + \eps} \approx \tg \cos \angle x,
\end{align}
where $\eps>0$ is a small saturation constant. It is clear in this form that the input is a fed-back version of the ``phase'' only. We obtain the closed-loop stance dynamics
\begin{align}
\label{eq:fVH}
\dot x = \hyb{f}{v}_1(x) := - \omega \J x + \left(-2\bar\beta\omega + \tfrac{k_t}{\omega(\Vert x \Vert + \eps)} \right) x_2 e_2.
\end{align}

\begin{figure}[t]
% \minipage[b]{0.3\columnwidth}
% \centering
% \def\svgscale{0.42}\import{../fig/}{rccPlot.pdf_tex} 
% \endminipage\hfill
% \minipage[b]{0.7\columnwidth}
% \centering
% \def\svgscale{0.75}\import{../fig/}{vhHybridSims.pdf_tex} 
% \endminipage\hfill
\minipage[b]{\columnwidth}
\centering
\iftoggle{techrep}{
\def\svgwidth{0.9\columnwidth}
}{
\def\svgscale{0.4}
}
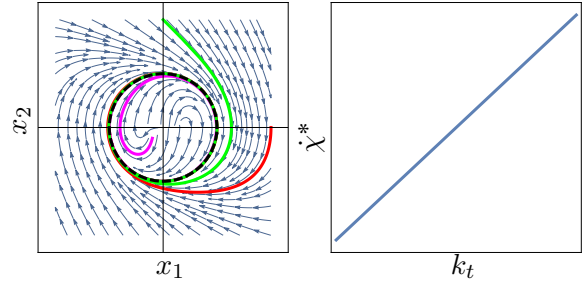
\endminipage\hfill
\caption{{\bf Left:} The vector field and an execution of (\ref{eq:fVH}), showing a stable limit cycle. {\bf Right:} The vertical ``energy'' is easy to tune with $k_t$.}
% \fxerror{R1: The axis labels in Figure 3 need some
% revision/explanation. What are the axes of the left
% subfigure?}
% \fxerror{R2: Fig 3 is unclear}
\label{fig:rccPlot}
\end{figure}

\begin{prop}[Oscillatory energization stability]
\label{prop:NegativeDamping}
The vertical hopping template (\ref{eq:fVH}) has a unique attracting periodic orbit.
\end{prop}

\begin{proof}
\iftoggle{techrep}{
First, note that $x = 0$ is the only equilibrium of~\eqref{eq:fVH}. Secondly, note that 
\begin{align}
x^T \dot x = x_2^2 \left( -2\bar\beta \omega + \tfrac{k_t}{\omega(\Vert x \Vert + \eps)}\right),
\end{align}
which is zero on the set $\Vert x \Vert^* = \tfrac{k_t}{2\bar\beta \omega^2} -\eps$. Additionally, since $x^T \dot x\restr{\Vert x \Vert < \Vert x \Vert^*} > 0$ and $x^T \dot x\restr{\Vert x \Vert > \Vert x \Vert^*} < 0$, this limit cycle is attracting. 
% 
% We can find a trapping region, $\sR_o := \{x : \Vert x \Vert \le \tfrac{\tg}{\sigma} \}$ (shown in Fig. \ref{fig:rccPlot}). To see why, note that
% \begin{align}
% x^T \hyb{f}{v}_1(x) &\le -\sigma x^T x + \tfrac{\tg x^T x}{\Vert x \Vert},
% \end{align}
% i.e. $\Vert x \Vert \ge \frac{\tg}{\sigma} \implies \tfrac{d}{dt}(x^T x) \le 0$.
% 
% Next, note that 
% \begin{align}
% \D \hyb{f}{v}_1 \restr{x=0} = \mat{(-1+\frac{\tg}{\eps})\sigma & (1+\frac{\tg}{\eps})\omega \\ -\omega & -\sigma},
% \end{align}
% for which $\tr \D \f{1}{v} \restr{x=0} = \sigma \cdot (\frac{\tg}{\eps} - 2) > 0$ for a sufficiently small $\eps$, and $\det \D \f{1}{v} \restr{x=0} = 1 + \frac{\tg}{\eps}(\omega^2 - \sigma^2) >0$, by assumption. Thus, $x = 0$ is a source.
% 
% Now we know that the trapping region is annular, and contains no equilibria. Applying the Poincare-Bendixson Theorem~\cite{guckenheimer_nonlinear_1990}, we conclude that the limiting trajectory is a closed orbit lying inside the trapping region.
}{
Included in~\cite{CompositionHoppingTR}.
}
\end{proof}

% While this result guarantees periodic hopping behavior, we cannot yet preclude multiple limit cycles, some of which may not be attracting. Encouraged by simulation results in Fig.\ \ref{fig:rccPlot}, we make the following conjecture.
% \begin{conjecture}
% \label{ConjUniqueLimitCycle}
% The attracting limit cycle of (\ref{eq:fVH}) is unique and nondegenerate.
% \end{conjecture}

\iftoggle{techrep}{
Writing $x(t, x_0)$ to denote the flow generated by~\eqref{eq:SpringMass}, and letting $\sS(x_0) := \min \{t >0 \mid e_1^T x(t,x_0) = 0 \}$ denote the stance time (since $x_1$, vanishes exactly at the liftoff), we define  the vertical stance map,
\begin{align}
\hyb{F}{v}_1( \dot \chi) := e_2^T x( \sS(\dot \chi,0), (\dot \chi, 0)).
\end{align}
}{}

As a corollary to Proposition \ref{prop:NegativeDamping}, we know $\hyb{F}{v}_1$ 
\iftoggle{techrep}{}{(the vertical stance map, cf.\ \cite{CompositionHoppingTR})} 
has an asymptotically stable fixed point, $\dot\chi^*$, and
$-1 < \D \F{1}{v} \restr{\dot\chi^*} < 1$.

% \fxwarning{FIXME: uniqueness of limit cycle? Fallback plan is to leave as conjecture, but still working on it.}

Ballistic flight simply reverses the velocity,
\begin{align}
\label{eq:IsoFlight}
\F{2}{v}(\dot\chi) := -\dot\chi.
\end{align}
Note that by symmetry ($\hyb{f}{v}_1$, and consequently $\hyb{F}{v}_1$ are odd), 
$\hyb{F}{v}_1 \circ \hyb{F}{v}_1 = \hyb{F}{v}_2\circ\hyb{F}{v}_1\circ\hyb{F}{v}_2\circ\hyb{F}{v}_1$,
i.e. the stability properties of the hybrid system are the same as that of the stance map as analyzed in Proposition \ref{prop:NegativeDamping}.
Define
\begin{align}
\label{eq:vgDefn}
\vg = \gvg(\dot\chi) := \tfrac{-\F{1}{v}(\dot\chi)}{\dot\chi},
\end{align}
the effective coefficient of restitution through stance, or the so-called ``velocity gain'' during SLIP stance \cite{arslan_reactive_2012}. Note that there is a unique fixed point, $\vg^* = 1$, in these coordinates, which is necessary and sufficient for the smooth invertibility of $\gvg$, as can be seen by direct computation of its derivative.

% Note that 
% \begin{align*}
% \D \gvg &= - \tfrac{\D \F{1}{v}}{\dot\chi} + \tfrac{\F{1}{v}(\dot\chi)}{\dot\chi^2} \\
% \implies \D \gvg \restr{\dot\chi^*} &= -\tfrac{1}{\dot\chi^*} \left( 1 + \D \F{1}{v} \restr{\dot\chi^*} \right) < 0,
% \end{align*}
% where the inequality is guaranteed by assuming Conjecture \ref{ConjUniqueLimitCycle}, and this monotonicity allows us to infer the existence of an inverse, $\gvg^{-1}$, and, hence, to introduce $\gvg$ as a convenient change of coordinates for the purely vertical component of the hopper in some neighborhood of the fixed point, $\dot\chi^*$.
Conjugating the  touchdown velocity return map via this
diffeomorphism, we can define a return map for $\vg$, $\hyb{F}{v}$,
% Applying (\ref{eq:IsoFlight}) and (\ref{eq:vgDefn}),
\begin{align}
\label{eq:FvgDefn}
\Fvg(\vg) := \gvg\circ\F{2}{v}\circ\F{1}{v}\circ\gvg^{-1}(\vg)
% &= \gvg(- \F{1}{v}(\gvg^{-1}(\vg))) = \gvg(\gvg(\gvg^{-1}(\vg)) \cdot \gvg^{-1}(\vg)) \nonumber\\
= \gvg(\vg \gvg^{-1}(\vg)).
\end{align}

\begin{prop}[Vertical stability]
\label{prop:RadialStability}
The velocity gain return map, $\hyb{F}{v}$, has an asymptotically stable fixed point, $\vg^* := 1$, and $\D \hyb{F}{v} \restr{\vg=1} = - \D \hyb{F}{v}_1 \restr{\dot\chi^*}$.
\end{prop}

\begin{proof}
\iftoggle{techrep}{
This directly follows from the observation that $\vg$ and touchdown velocity are related by a diffeo, Proposition \ref{prop:NegativeDamping}, and the simple form of $\hyb{F}{v}_2$ in (\ref{eq:IsoFlight}).
}{
Included in~\cite{CompositionHoppingTR}.
}
% Directly taking a gradient of (\ref{eq:FvgDefn}),
% \begin{align*}
% \D \Fvg &= \D \gvg \restr{\vg \gvg^{-1}(\vg)}\left( \gvg^{-1}(\vg) + \frac{\vg}{\D \gvg}\right) \\
% \implies \D \Fvg \restr{\vg=1} &= 1 + x_1^* \cdot \D \gvg \restr{x_1^*} = -\D \F{1}{v} \restr{x_1^*}.
% \end{align*}
% Thus, the velocity gain $\vg$ is an alternate coordinate which is sufficient to analyze the stability of the vertical hopping subsystem.
\end{proof}

%%%%%%%%%%%%%%%%%%%%%%%%%%%%%%%%%%%%%%%%%%%%%%%%%%%%%%%%%%%%%%%%%%%%%%%%%%%%%%%%

\subsection{Controlled Fore-Aft Speed (1DOF)}
\label{sec:FA}

Running and walking systems of a large variety from the sagittal or frontal plane resemble inverted pendula during stance \cite{holmes_dynamics_2006}, usually controlled by stepping strategies. It has been shown that a fixed touchdown angle can admit a reasonable basin of stability around an emergent attracting steady-state velocity in SLIP \cite{ghigliazza_simply_2005}. The capture point \cite{pratt_capture_2006} and zero moment point \cite{kajita_biped_2003} methods use a quasistatic heuristic which is related to these ideas, but are not explicitly designed to servo to desired nonzero speeds.
We attempt here to place the empirical success of \cite{raibert_legged_1986} in the context of a model where its stability properties can be analyzed.

%%%%%%%%%%%%%%%%%%%%%%%%%%%%%%%%%%%%%%%%%%%%%%%%%%%%%%%%%%%%%%%%%%%%%%%%%%%%%%%%

\subsubsection{The Raibert Stepping Controller}
\label{sec:RaibertController}

In his classical empirical study, Raibert \cite{raibert_legged_1986} inspired decades of subsequent experimentation and analysis by offering the following observations\footnote{These conditions are not a direct result of SLIP's nonlinear dynamics, but are applicable to regime of interest.}
about the pendular stance phase in his running machine travelling at forward speed, $\dot x$, and stepping with a touchdown angle $\beta(\dot x)$ (as in Fig.\ \ref{fig:MBHop}):

\begin{assumption}[Raibert observations]
\label{def:RaibertAssumptions}
\begin{inparaenum}[(i)]
\item For each speed, $\dot x$, there is a neutral\footnote{In this context, ``neutral'' means $\dot x^+ = \dot x$, where $\dot x^+$ refers to the fore-aft speed at the subsequent touchdown event.} touchdown angle, $\beta^*(\dot x)$
\item this neutral angle is monotonic with speed, $\D_{\dot x} \beta^* > 0$, and 
\item deviations from touchdown angle cause negative acceleration, i.e. $\D_\beta (\dot x^+ - \dot x) \restr{\beta=\beta^*} < 0$.
\end{inparaenum}
\end{assumption}

\begin{prop}[Raibert stepping controller]
\label{prop:RaibertController}
Under assumptions \ref{def:RaibertAssumptions}(i-iii), the Raibert stepping controller,
\begin{align}
\label{eq:RaibertController}
\beta : \dot x \mapsto \beta^*(\dot x) + k_p(\dot x - \dot x^*)
\end{align} stabilizes the forward speed to ${\dot x}^*$.
\end{prop}

\begin{proof}
\iftoggle{techrep}{
Note that 
\begin{align*}
\D_{\dot x} (\dot x^+ - \dot x) &= \D_\beta (\dot x^+ - \dot x) \cdot \D_{\dot x} \beta(\dot x) \\
&= \D_\beta ({\dot x}^+ - {\dot x}) \cdot (D_{\dot x} \beta^* + k_p) \\
\implies \D_{\dot x} {\dot x}^+ \restr{{\dot x} = {\dot x}^*} &= 1 + \D_\beta({\dot x}^+ - {\dot x}) \cdot (D_{\dot x} \beta^* + k_p).
\end{align*}
From the sign properties of various terms, we note that for small $k_p$, $-1 < \D_{\dot x} {\dot x}^+ < 1$.
}{
Included in~\cite{CompositionHoppingTR}.
}
\end{proof}
% Surprisingly, notwithstanding the large literature arising from his pioneering ideas, to our knowledge, these simple observations represent the first stability analysis of any representation of Raibert's stepping controller. 
% In order to establish a compositional view of SLIP, we will adopt a simple model of pendular stance that is in harmony with the Raibert assumptions, thereby automatically providing sufficient conditions for stability of the isolated fore-aft template closed-loop dynamics $\hyb{F}{fa}$ defined in (\ref{eq:BHopIso})\footnote{Of course, the sufficiency of the Raibert stepping controller (\ref{eq:RaibertController}) for stability of the composed coupled systems is not automatic and represents a central contribution of the paper. This is established for the intermediate SLIP template, $\hyb{F}{s}$,  in Proposition \ref{prop:SLIPComposition}.}.

%%%%%%%%%%%%%%%%%%%%%%%%%%%%%%%%%%%%%%%%%%%%%%%%%%%%%%%%%%%%%%%%%%%%%%%%%%%%%%%%

\begin{figure}[t]
\minipage[b]{\columnwidth}
\centering
\iftoggle{techrep}{
\def\svgwidth{\columnwidth}
}{
\def\svgscale{2.3}
}
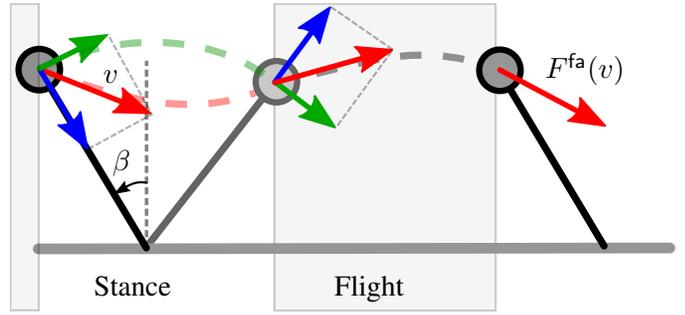
\endminipage\hfill
\caption{A simple model for the 1DOF fore-aft dynamics in SLIP, closely related to BHop \cite{arslan_reactive_2012}.}
\label{fig:MBHop}
\end{figure}

%%%%%%%%%%%%%%%%%%%%%%%%%%%%%%%%%%%%%%%%%%%%%%%%%%%%%%%%%%%%%%%%%%%%%%%%%%%%%%%%

\subsubsection{Modified BHop as a Fore-Aft Model}
\label{sec:MBHop}

Building on existing SLIP literature \cite{schwind_control_1995}, we make the following assumptions about pendular stance:
\begin{assumption}[Pendular stance]
\label{def:PendularAssumptions}
During stance,
\begin{inparaenum}[(i)]
\item the effects of gravity are negligible\footnote{We suspect that the less restrictive Geyer approximation~\cite{geyer_spring-mass_2005} is sufficient, but leave this generalization to future work.} compared to spring potential / damping forces, 
\item radial deflections are negligible, \label{ConstantAngVel}
\item time of stance is constant, \label{ConstantStanceTime} and
\item the angle swept by the leg admits a small-angle approximation. \label{StanceSmallAngle}
\end{inparaenum}
\end{assumption}
Schwind \cite{schwind_control_1995} approximated that angular momentum about the toe is constant during stance, but we simplify further with the second assumption, and conclude that the angular velocity is roughly constant during stance. We adopt the third approximation from Raibert \cite{raibert_legged_1986}, and the last approximation is made for the ensuing analytical simplifications in \S\ref{sec:PhysicalDecoupling}, but we find empirically (\S\ref{sec:Experiments}) that it is not critical in practice.

% \fxerror{DAN: While I'd really like to not have to make assumption (iii), I don't yet know how to fix these two Propositions to go through without it. If you think relaxing this assumption is relatively critical, please let me know and I'll prioritize it over hardware work.}

These assumptions lead directly to the construction of the following return map acting on touchdown velocity in Cartesian coordinates (cf.\ Fig.\ \ref{fig:MBHop}). Then,
\begin{align}
\Fr{s}(v,\vg) &= \mat{1 & \\ & -1} \R(-\gamma+\beta) \mat{1 & \\ & -\vg} \R(-\beta) v \nonumber \\
&= \R(\gamma - \beta) \mat{1 & \\ & \kappa} \R(-\beta) v\label{eq:MBHop},
\end{align}
where $\vg$ (explicitly, the interaction from the radial component of SLIP) is taken to be a fixed parameter at this stage, $\gamma(v_1) \approx \frac{v_1 T_s}{\rho_l}$
is the angle swept by the leg over the course of stance and $\beta(v_1)$ is the leg touchdown angle (\S\ref{sec:RaibertController}). 
This model is only a slight modification\footnote{Specifically, the similarities are apparent between (\ref{eq:MBHop}) and (19) of \cite{arslan_reactive_2012}. The slightly discrepancy should be attributed to our insistence on using the physical touchdown and sweep angles $\beta$ and $\gamma$ in the model, whereas the abstract parameter $\theta$ in \cite{arslan_reactive_2012} results in a more succinct form.} of BHop \cite{arslan_reactive_2012}.

This analytically tractable model 
\begin{inparaenum}[(i)]
\item allows us to ``separate'' the radial dynamics (encapsulated in $\vg$) from the contributions of the fore-aft model itself,
\item captures the exchange of vertical and horizontal energy through stepping, and
\item matches the empirically observed Raibert conditions (Fig. \ref{fig:Fbhcontour}) as well as empirical data (Fig.\ \ref{fig:compositionPlots}), suggesting it is physically applicable and not just an analytical convenience.
\end{inparaenum}

For now we restrict our attention to $\vg = 1$, and generalize to include the radial dynamics in \S\ref{sec:SLIP}. With this restriction,
\begin{align}
\label{eq:BHopIso}
\Fbhiso(v) := \Fr{s}(v,1) = \R(\gamma - 2\beta) v,
\end{align}

While we choose to parameterize the return map as a function of $v\in\bbR^2$, it is really a 1D map:
% This is explored in more depth in \S\ref{sec:SLIP}, but it can be immediately seen in (\ref{eq:BHopIso}) that $\Vert \Fbhiso(v) \Vert = \Vert v \Vert$, i.e. the magnitude of $v$ is invariant to $\Fbhiso$.

\begin{prop}[Fore-aft stability]
\label{prop:ForeAftStability}
MBHop with the Raibert controller presents a stable touchdown return map.
\end{prop}

\begin{proof}
\iftoggle{techrep}{
We can check that $\Fbhiso$ satisfies each of the Raibert conditions (Fig. \ref{fig:Fbhcontour}), thereby concluding automatically from Proposition \ref{prop:RaibertController} that the Raibert controller will ensure local stability.

Alternatively, the utility of our simple analytical model (\ref{eq:MBHop})-(\ref{eq:BHopIso}) is that we can directly compute the stability properties under the Raibert controller (\ref{eq:RaibertController}),
\begin{align}
\D \Fbhiso(v) := \R + \J \R v \cdot (\D\gamma - 2\D\beta) e_1^T,
\end{align}
where $\R$ is evaluated at $\gamma - 2\beta$. By inspection, the (desired) fixed point of (\ref{eq:BHopIso}) is $\beta = \gamma/2$ (this is the neutral touchdown angle). Evaluated at the fixed point,
\begin{align}
\label{eq:FbhisoEval}
\D \Fbhiso(v^*) = I - 2 k_p \J v^* e_1^T = \mat{1 + 2 k_p v_2^* & 0\\- 2 k_p v_1^* & 1},
\end{align}
which is lower-triangular. The eigenvalues are $\{1, 1+2 k_p v_2^*\}$, which capture the local stability of the single fore-aft DOF ($1 + 2 k_p v_2 < 1$) as well as the degeneracy of the map.

To see why the last statement is true, note that we can find a rank 1 map \[
\iota : \bbR^2\rightarrow\bbR : v \mapsto \Vert v \Vert,
\] which is invariant to $\Fbhiso$, i.e. $\iota \equiv \iota\circ\Fbhiso$. Taking a gradient of both sides and using the chain rule, \[
\D \iota \restr{v} = \D \iota \restr{\Fbhiso(v)} \cdot \D \Fbhiso \restr{v}.
\] Evaluating at the fixed point $v^*$, \[
\D \iota \restr{v^*} = \D \iota \restr{v^*} \cdot \D \Fbhiso \restr{v^*},
\] i.e. $\D \iota\restr{v^*}$ is a left eigenvector of $\D \Fbhiso \restr{v^*}$ with unity eigenvalue.

Consequently, under iterations of this map, we get an invariant submanifold spanned by the orthogonal complement of the unity eigenvector, resulting in a ``dimension reduction'' (to a codimension 1 submanifold). In our case, $\Fbhiso$ is really a 1D map, even though its (co)domain in $\bbR^2$.
}{
Included in~\cite{CompositionHoppingTR}.
}
\end{proof}

%%%%%%%%%%%%%%%%%%%%%%%%%%%%%%%%%%%%%%%%%%%%%%%%%%%%%%%%%%%%%%%%%%%%%%%%%%%%%%%%

\begin{figure}[t]
\centering
\iftoggle{techrep}{
\includegraphics[width=\columnwidth]{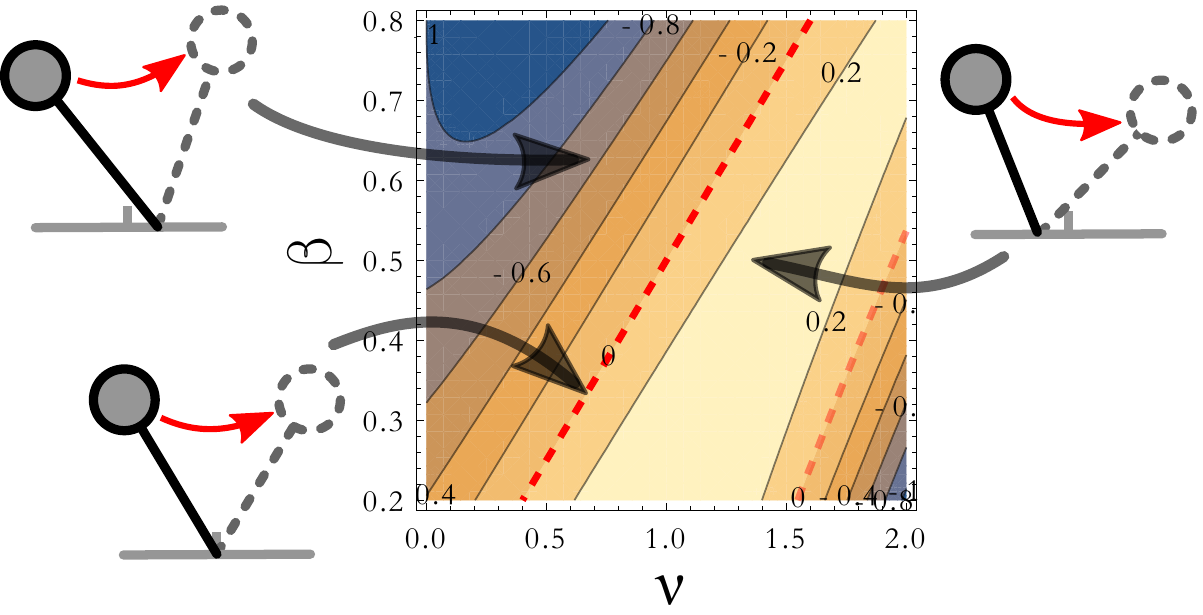}
}{
\includegraphics[width=7cm]{./FbhcontourCartoons}
}
\caption{A contour plot of the fore-aft acceleration ${\dot x}^+-{\dot x}$ produced by the MBHop model for a range of fore-aft speed ${\dot x}$ and touchdown angle $\beta$. This plot depicts that (in a range around the neutral angle), this model captures all the conditions of assumption \ref{def:RaibertAssumptions}.}
\label{fig:Fbhcontour}
\end{figure}

%%%%%%%%%%%%%%%%%%%%%%%%%%%%%%%%%%%%%%%%%%%%%%%%%%%%%%%%%%%%%%%%%%%%%%%%%%%%%%%%

\subsection{SLIP as a Parallel Composition}
\label{sec:CouplingSLIP}

In order to anchor our 1DOF templates in the classical SLIP model (2DOF point mass with 2DOF springy leg), we simply ``play back'' our devised control schemes (Sections \ref{sec:Radial} and \ref{sec:FA}). In the following subsections, we check that the closed-loop executions in the higher-DOF body still resemble a cross-product of our template behaviors.
% We emphasize that when moving to a more complex body, we do not add complexity to our controllers, rather, the closed-loop templates which we couple together are intrinsically amenable to parallel execution with minimal leakage of energy, by design. 
For instance, prior literature has observed a decomposition of SLIP dynamics into radial and tangential components, but to our knowledge there is no complete account of the stability of the parallelly composed (closed-loop) templates in these components.

%%%%%%%%%%%%%%%%%%%%%%%%%%%%%%%%%%%%%%%%%%%%%%%%%%%%%%%%%%%%%%%%%%%%%%%%%%%%%%%%

\subsubsection{Hybrid Dynamical Model of SLIP}
\label{sec:SLIPHybrid}

We will construct our template plant model from \cite{schwind_control_1995}: a bead of mass 1 at (Cartesian) coordinates $(\hyb{x}{s}, \hyb{z}{s})\in\bbR^2$, with a springy (Hooke's law spring constant $k_s$) massless leg of length\footnote{We use $\theta$ for leg ``joints'' to be consistent with \cite{johnson_legged_2013}.} $\hyb{\theta_2}{s}\in\bbR_+$ (where $\bbR_+$ is restricted to \emph{strictly} positive reals, and is open) and rest length $\rho_l$, at an angle of $\hyb{\theta_1}{s}\in S^1$ from vertical. Let $\hyb{\smq}{s} := (\hyb{\theta_1}{s},\hyb{\theta_2}{s},\hyb{x}{s}, \hyb{z}{s})$. Using assumption \ref{def:PendularAssumptions}(\ref{StanceSmallAngle}) as a convenience (though that assumption is not required for this formulation), the touchdown and lift-off conditions can be specified in terms of the zeros of $\hyb{\sma}{s} := \hyb{z}{s} - \rho_l$.

Define $\hyb{\sQ}{s}_i := S^1\times\bbR_+\times\bbR\times\sI_i$, where $\bbR = \sI_1 \disjunion \sI_2 := (-\infty,\rho_l] \disjunion (\rho_l,\infty)$. Then, $\hyb{\sD}{s}_i := T\hyb{\sQ}{s}_i$, and
\begin{align}
\hyb{f}{s}_1(\hyb{\smq}{s}, \dot{\hyb{\smq}{s}}) &:= \left(\dot {\hyb{\smq}{s}},
\mat{ %
-\frac{2 \dot {\hyb{\theta_1}{s}}\dot {\hyb{\theta_2}{s}}}{{\hyb{\theta_2}{s}}} \\
{\hyb{\theta_2}{s}} \dot {\hyb{\theta_1}{s}}^2 + k_s(\rho_l - {\hyb{\theta_2}{s}}) \\ 
\star
}\right), \label{eq:SLIPStance} \\
\hyb{f}{s}_2(\hyb{\smq}{s}, \dot{\hyb{\smq}{s}}) &:= \left(\dot {\hyb{\smq}{s}}, \mat{\star\\0\\-g}\right), \label{eq:SLIPFlight}
\end{align}
where the unspecified components are
\begin{inparaenum}[(i)]
\item the mass-center dynamics which are constrained by $\mat{{\hyb{x}{s}}\\{\hyb{z}{s}}} = \hyb{\theta_2}{s} \mat{{-\sin \hyb{\theta_1}{s}} \\ {\cos \hyb{\theta_1}{s}}}$ in (\ref{eq:SLIPStance}), and
\item the degenerate massless leg dynamics in (\ref{eq:SLIPFlight}).
\end{inparaenum}

\iftoggle{techrep}{
\paragraph{The Guard Set is $\partial \hyb{\sD}{s}$}

Since $\hyb{\sQ}{s}$ is itself a cross product of Euclidean spaces and Lie groups, we can identify the tangent bundle with a cross product, $T\hyb{\sQ}{s}_i \approx \hyb{\sQ}{s}_i \times \bbR^4$. Then, the boundary of the product space only contains parts from $\sI_i$, which corresponds exactly to the zeros of $\hyb{\sma}{s}$ (\S\ref{sec:SLIPHybrid}).

\paragraph{Reset Maps}

Let us define the functions 
\begin{align}
\Cart &: S^1 \times \bbR_+\rightarrow \bbR^2 : \mat{\theta_1\\\theta_2} \mapsto \theta_2\mat{-\sin\theta_1\\\cos\theta_1} \\
\Pol &: \bbR^2 \rightarrow S^1 \times \bbR_+ : u \mapsto \mat{\angle u \\ \Vert u \Vert}.
\end{align}

The reset maps are defined as 
\begin{align*}
\hyb{r}{s}_1 &: \hyb{\sD}{s}\rightarrow\hyb{\sD}{s} : \mat{\theta\\\dot\theta\\\mat{x\\z}\\\mat{\dot x\\\dot z}} \mapsto \mat{\theta\\\dot\theta\\\Cart(\theta)\\D\Cart\restr{\theta}\cdot \dot\theta},\\
\hyb{r}{s}_2 &: \hyb{\sD}{s}\rightarrow\hyb{\sD}{s} : \mat{\theta\\\dot\theta\\\mat{x\\z}\\\mat{\dot x\\\dot z}} \mapsto \mat{\Pol\left(\mat{x\\z}\right)\\D\Pol\cdot\mat{\dot x\\\dot z}\\\mat{-z\tan\beta(\dot x)\\z}\\\mat{\dot x\\\dot z}}.
\end{align*}
}{
We explicitly write the guard set and reset map in~\cite{CompositionHoppingTR}.
}

% We're tight for space but having once promised a formal hybrid system definition we are honor bound to specify carefully the  reset map as a function over  the entire guardset, which in turn we promised would be exactly the boundary of the domain, i.e., 
% \begin{align}
% \def \Ds {mathcal{D}}
% \def \C { \S^1 \times \Real }
% \def \Rlp {\Real_= }
% \def \Int_1 { (- \infty, \rho_l)}
% \def \Int_2 { (\rho_l, \infty)}
% %
%  \partial \Ds_i  = 
% %
% \C \times \Rlp \times  \coprod
% \C \times Closure(\Int_i) \times \partial \Rlp \coprod
% \C \times  \times \partial \Rlp times \partial \Int_i
% \end{align}

% Perhaps the thing to do is put this into an appendix which might be jettisoned from the publication version of the paper, once accepted. 

%%%%%%%%%%%%%%%%%%%%%%%%%%%%%%%%%%%%%%%%%%%%%%%%%%%%%%%%%%%%%%%%%%%%%%%%%%%%%%%%

\subsubsection{Anchoring the 1DOF Templates}
\label{sec:AnchoringTemplates}

Consequent upon the above model---where each hybrid mode is dynamically 2DOF---SLIP is a 4D dynamical system (one parameterization being $(x,z,v)$, where $v\in\bbR^2$ is the touchdown velocity, and $(x,z)\in\bbR^2$ is the Cartesian location of the point mass at touchdown). The efficacy of our 2D return map analysis is established by arguments similar to those of \cite{schmitt_mechanical_2000}: the Poincare section $z^\mathrm{TD} = \rho_l \cos\beta(v)$ eliminates one dimension, and the equivariance of the dynamics with $x$ eliminates another.

We first observe that our MBHop model of \S\ref{sec:MBHop} still represents the pendular stance correctly under assumption \ref{def:PendularAssumptions}. However, $\vg$ is not a fixed parameter, but evolves according to dynamics similar to $\hyb{F}{v}$ in Proposition \ref{prop:RadialStability}.
From \eqref{eq:RaibertController} and \eqref{eq:MBHop}, the embedded $(\vg = 1, v= v^*)$ submanifold is invariant. We show in Proposition \ref{prop:SLIPComposition} that it is also attracting.

Let us define $\hw : \bbR^2 \rightarrow \bbR^2$ as
\begin{align}
w = \hw(v) := \R(-\beta(v)) v.
\end{align}

\iftoggle{techrep}{
\begin{lemma}
Let $\sV := \{v \in \bbR^2: v_2 < -\frac{2\rho_l}{T_s} \}$. Then $\hw\restr{\sV}$ is a local diffeomorphism.\footnote{Physically, the restriction to $\sV$ means that the hopper must have sufficient vertical component of touchdown velocity, essentially eliminating ``grazing'' ground impacts.}
\end{lemma}

\begin{proof}
Note that
\[
\D \hw = \R - \J \R v \D\beta e_1^T,
\]
where $\R$ is understood to be evaluated at $-\beta(v)$.
By inspection, $\D \hw$ could only have a test vector $\R^T \J \R v$ in its kernel, i.e.
\begin{align*}
\D \hw \cdot (\R^T \J \R v) = (1 - \D\beta e_1^T \R^T \J \R v) \J \R v \neq 0,
\end{align*}
since we know $v\neq 0$, $\D\beta = \left(\tfrac{T_s}{2\rho_l} + k_p\right)$ and so
\begin{align*}
1 - \D\beta e_1^T \R^T \J \R v = 1 + v_2\left(\tfrac{T_s}{2\rho_l} + k_p\right) < 0,
\end{align*}
by the conditions assumed on $k_p$. Thus $\D \hw$ is nonsingular, and $\hw$ is a local diffeo.
\end{proof}
The vector $w$ gives a tangential/radial decomposition of $v$ (i.e. polar with respect to the leg angle).
}{
This mapping is a local diffeomorphism (cf.~\cite{CompositionHoppingTR}), and the vector $w$ gives a tangential/radial decomposition of $v$ (i.e. polar with respect to the leg angle).
}

Additionally, using (\ref{eq:vgDefn}), we can ``recover'' the $\vg$-dynamics in the coupled system: $\vg = \gvg(w_2)$. We prefer the redundant $(v, \vg)$ parameterization because of analytical tractability.

\begin{prop}[Stability of SLIP as a composition]
\label{prop:SLIPComposition}
For 
\begin{inparaenum}[(i)]
\item stable vertical hopping with $-1 + \epsvh < -\D\F{1}{v}\restr{*} < 1 - \epsvh$,
\item sufficiently\footnote{Formally, this means that $k_p$ can be chosen as a function of $\epsvh$.} small $k_p$ in the Raibert contoller,
\end{inparaenum}
parallel composition of the radial and fore-aft templates results in a locally stable 2D return map, $\hyb{F}{s}$.
\end{prop}

\begin{proof}
\iftoggle{techrep}{
We choose to perform our stability analysis at a section just after touchdown (in $w = \hw(v)$ coordinates). From (\ref{eq:MBHop}), the return map in $w$-coordinates is
\begin{align*}
\Fw(w) &:= \hw \circ \Fr{s} \circ \hw^{-1}(w)\restr{\vg = \gvg(w_2)} \\
&= \R(\eta(w)) \mat{1 &\\&\gvg(w_2)} w,
\end{align*}
where $\eta := (\gamma - \beta - \beta\circ \Fr{s})\circ \hw^{-1}$. Now,
\begin{align*}
\D \Fw &= \D_w \Fw + \D_\vg \Fw \cdot \D \gvg e_2^T,
\end{align*}
where the first summand can be thought of as loosely the isolated fore-aft subsystem behavior, and the second summand is the perturbation from the radial subsystem. We will evaluate this quantity at the fixed point $w^* = \hw(v^*)$.

Observe that using (\ref{eq:RaibertController}), $\D\eta\restr{*} = -2 k_p e_1^T \D \hw^{-1}$. Proceeding just like in Proposition \ref{prop:RadialStability}, 
\begin{align*}
\D \gvg \restr{*} &= -\frac{1}{w_2^*} \left(1 + \D\F{1}{v}\restr{w_2^*}\right), \\
\D_\vg\Fw &= \R(\eta) e_2 e_2^T w \implies \D_\vg\Fw\restr{*} = w_2^* e_2.
\end{align*}
Lastly, the ``isolated'' term computes similar to (\ref{eq:FbhisoEval}),
\begin{align*}
\D_w \Fw &= \R \mat{1 &\\&\vg} + \J \R \mat{1 &\\&\vg} w \D\eta, \\
\implies \D_w \Fw \restr{*} &= I + \J w^* \D\eta\restr{*}.
\end{align*}

Putting all of these together, 
\begin{align*}
\D \Fw\restr{*} = \mat{1&\\&-\D\F{1}{v}\restr{*}} + p q^T,
\end{align*}
where $p := - 2k_p \J w^*, q^T := e_1^T \D \hw^{-1}$. Using the matrix determinant lemma,
\begin{align*}
\tr \D\Fw &= 1 -\D\F{1}{v}\restr{*} + p^T q\\
\det \D\Fw &= -\D\F{1}{v}\restr{*}\left(1 - q^T \mat{1&\\&-\D\F{1}{v}\restr{*}^{-1}} p\right).
\end{align*}

Now notice that since $\D \hw$ is well-conditioned, we can claim an upper bound on
\begin{align*}
\vert p^T q \vert \le 2 k_p \Vert \J w^* \Vert \Vert \D \hw^{-1} \Vert \le k_p \Xi.
\end{align*}
Also, the quadratic form $q^T \mat{1&\\&-\D\F{1}{v}\restr{*}^{-1}} p$ must have
\begin{align*}
\left\vert q^T \mat{1&\\&-\D\F{1}{v}\restr{*}^{-1}} p \right\vert \le \vert p^T q \vert,
\end{align*}
since $\mat{1&\\&-\D\F{1}{v}\restr{*}^{-1}}$ has norm less than 1.

It can be checked that both eigenvalues are of absolute value bounded by unity iff all of
\begin{inparaenum}[(i)]
\item $\det < 1$,
\item $\det > \tr-1$, and
\item $\det > -\tr-1$
\end{inparaenum}
are true. These inequalities follow from condition (ii) of Proposition \ref{prop:SLIPComposition} and choosing small enough $k_p$ such that $2 k_p \Xi < \epsvh$.
}{
Included in~\cite{CompositionHoppingTR}.
}
\end{proof}

% \fxwarning{(If there is time) use simulation to test Proposition \ref{prop:SLIPComposition}, i.e. we made a conservative assumption about $k_p$ and in Assumption \ref{def:PendularAssumptions}.}

%%%%%%%%%%%%%%%%%%%%%%%%%%%%%%%%%%%%%%%%%%%%%%%%%%%%%%%%%%%%%%%%%%%%%%%%%%%%%%%%

\begin{figure*}[t]
\minipage[b]{\textwidth}
\centering
\iftoggle{techrep}{
\def\svgwidth{0.6\textwidth}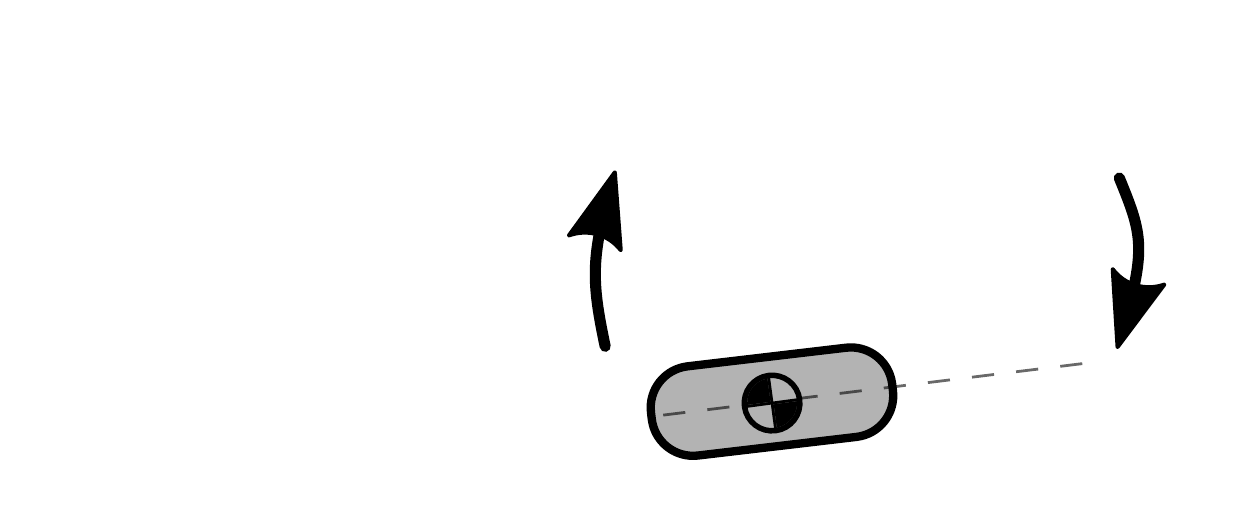 
}{
\def\svgscale{0.55}\import{./}{HIR.pdf_tex}
}
\endminipage\hfill
\caption{A hybrid 2DOF inertial reorientation template with two segments pinned at the CoM and no gravity. {\bf Left:} the net angular momentum of the system is constant. {\bf Right:} the system can correct the net angular momentum using reaction torques on the main body segment, but the tail DOF is subject to an unmodeled disturbance $\lightning$, or $\delta$ in (\ref{eq:HIRbody}).}
\label{fig:HIR}
\end{figure*}

%%%%%%%%%%%%%%%%%%%%%%%%%%%%%%%%%%%%%%%%%%%%%%%%%%%%%%%%%%%%%%%%%%%%%%%%%%%%%%%%

\section{Hybrid Inertial Reorientation (2DOF)}
\label{sec:HIR}

Our decision to energize the hopping behavior with a tail leaves introduces a new actuated DOF whose tight  dynamical coupling to both the mass center and the body orientation dynamics requires its careful control throughout the locomotion cycle. Recent literature \cite{johnson_tail_2012} has seen the development of a 1DOF ``inertial reorientation'' template for correcting the ``shape'' coordinate in a two-link body experiencing free-fall (constrained by conservation of angular momentum). Raibert \cite{raibert_legged_1986} introduced a pitch stabilization mechanism relying on reaction torques from hip actuation during stance. In this paper, we adopt the approach of composing these templates for 2DOF stabilization of appropriately defined ``pitch'' and ``shape'' coordinates of a two-link body/tail model.

Since in the physical system the tail actuator, $\tau_2$, is unavailable for attitude control in stance (because it is being ``monopolized'' as the destabilizing energy source for the SLIP subsystem), and the Raibert pitch correction mechanism (using the hip actuator, $\tau_1$) is unavailable in flight (due to absence of ground reaction force), we present a hybrid inertial reorientation (HIR) template (Fig.\ \ref{fig:HIR}) as the simplest exemplar body on which this 2DOF template is anchored.

We omit the Lagrangian derivation for this familiar subsystem \cite{johnson_tail_2012}, but exploit the fact that when pinned at the CoM, the dynamics are second-order LTI with no Coriolis terms. We perform a change of coordinates (inverting the constant inertia tensor) to obtain the (decoupled) dynamics
\begin{align}
\mat{\ddot a_1\\\ddot a_2} &= \begin{cases}
  \mat{ \tau_1 \\ \delta } & =: \hyb{p}{a}_1(Ta, \tau_1) \quad \text{(stance)}, \\
  \mat{ 0 \\ \tau_2 } & =: \hyb{p}{a}_2(Ta, \tau_2) \quad \text{(flight)} \label{eq:HIRbody},
  \end{cases}
\end{align}
where $(a_1,a_2)$ are the ``pitch'' and ``shape'' coordinates, respectively, and $\delta$ is an unmodeled disturbance term (explicitly added here with an eye toward the use of tail for spring energization in the physical system). In (\ref{eq:HIRbody}) we have now represented HIR as two \emph{independent} subsystems on which two identical 1DOF templates will be anchored in parallel (albeit in alternating stages of the hybrid execution).

Taking advantage of the direct affordance (by which we mean that both of the two decoupled 1DOF systems are completely actuated in, one and then other, of the alternating modes of their hybrid dynamics), we employ a graph-error controller \cite{koditschek_adaptive_1987} as a type of reduction.
% \footnote{This can be thought of as a template-anchor relation \cite{full_templates_1999}, but we omit the details of the anchor dynamics here.}.
Since our reference first-order dynamics are just $\dot a_i = - k a_i$, the independent closed-loop 1DOF subtemplate vector fields, $\hyb{f}{p} : Ta_1 \mapsto \dot{T a_1}$ and $\hyb{f}{sh} : Ta_2 \mapsto \dot{T a_2}$, are defined as
\begin{align}
\label{eq:HIRCL}
\ddot a_i = - k_g(\dot a_i + k a_i) = -k_g k a_i - k_g \dot a_i,
\end{align}
where the gain $k_g$ is understood to be high enough to make the transients of the anchoring dynamics irrelevant.

%%%%%%%%%%%%%%%%%%%%%%%%%%%%%%%%%%%%%%%%%%%%%%%%%%%%%%%%%%%%%%%%%%%%%%%%%%%%%%%%

\subsection{Hybrid Dynamical Model of HIR}
\label{sec:HIRHybrid}

Since the isolated model does not have any intrinsic physical mechanism for transitioning between modes, we add an exogenous clock signal, $\psi_a \in S^1$ such that $\psi_a \in [0,\pi]$ represents stance, and the complement represents flight. In this paper we sidestep the issue of phase-synchronization for the various compartments, but simply use $\psi_a$ to ensure our gains our tuned properly for the timescales of the coupled system (Proposition~\ref{prop:TMComposition}).

Define $\hyb{\sD}{a} = T S^2 \times \{(0,\pi] \disjunion (\pi,2\pi]\}$. Now the closed-loop template dynamics, $\hyb{f}{a} : TS^2\times S^1 \rightarrow T(T S^2 \times S^1)$ can be specified as
\begin{align}
\label{eq:HIR}
\hyb{f}{a}_1(\mat{T a\\\psi_a}) &= \left[
\begin{smallmatrix}
  0 & I & 0\\
  \left[\begin{smallmatrix} -k_g k & 0 & -k & 0 \\ 0 & 0 & 0 & 0 \end{smallmatrix}\right] & 0 & 0\\
  0 & 0 & \omega_a
\end{smallmatrix}\right] \mat{T a\\\psi_a} + \mat{0_{3\times1}\\\delta\\0}, \nonumber\\
\hyb{f}{a}_2(\mat{T a\\\psi_a}) &= \left[
\begin{smallmatrix}
  0 & I & 0\\
  \left[\begin{smallmatrix} 0 & 0 & 0 & 0 \\ 0 & -k_g k & 0 & -k \end{smallmatrix}\right] & 0 & 0\\
  0 & 0 & \omega_a
\end{smallmatrix}\right] \mat{T a\\\psi_a},
\end{align}
the guards sets are $\partial \hyb{\sD}{a} = TS^2 \times \{\{\pi\} \disjunion \{2\pi\}\}$
and the reset maps $\hyb{r}{a}_i = \mathrm{id}$ simply modify the dynamics (\ref{eq:HIRbody}) at $\psi_a=\pi$ (stance to flight) and $\psi_a = 0$ (flight to stance).

%%%%%%%%%%%%%%%%%%%%%%%%%%%%%%%%%%%%%%%%%%%%%%%%%%%%%%%%%%%%%%%%%%%%%%%%%%%%%%%%

\subsection{HIR Stability Analysis}
\label{sec:HIRStability}

Let us denote $\bar\delta[i] := \int \delta dt$, the interval being over the stance phase of stride $i$. Also, define $\delmax = \max_t \bar\delta[t]$.

\begin{prop}[HIR Stability]
\label{prop:HIRStability}
Setting \[
k>\tfrac{2 \omega_a}{\pi}\log\left(1+\delmax/\eps_a\right)
\] results in the desired limiting behavior for $\hyb{F}{a}$: $\Vert a \Vert \rightarrow \sB_{\eps_a}(0)$, a neighborhood of 0 of size $\eps_a$.
\end{prop}

\begin{proof}
\iftoggle{techrep}{
Simply integrating the first-order dynamics (\ref{eq:HIR}), we get the touchdown return map $\Fr{a} : S^2 \rightarrow S^2$,
\begin{align}
\Fr{a}(a) = \zeta \cdot \left(a + \bar\delta \mat{0\\1}\right),
\end{align}
where $\zeta := e^{-k\pi/\omega_a} (1-k\pi/\omega_a)$. Iterating this return map, at stride $n\in\bbZ_+$,
\begin{align}
a[n] = \zeta^t a[0] + (\zeta^n \bar\delta[0] + \cdots + \zeta \bar\delta[n-1])e_2,
\end{align}
and using the triangle inequality, 
\begin{align}
\Vert a[n] \Vert \le \vert \zeta\vert^n \cdot \Vert a[0] \Vert + \delmax \left\vert\tfrac{\zeta}{1-\zeta}\right\vert.
\end{align}
Note that $\zeta < \frac{1}{1+\delmax/\eps_a}$ is a sufficient condition to ensure that $\Vert a[t] \Vert \le \eps_a$ asymptotically stable. Some algebra reveals that
\begin{align}
\label{eq:HIRkbound}
k > \tfrac{2 \omega_a}{\pi}\log\left(1+\tfrac{\delmax}{\eps_a}\right)
\end{align}
is, in turn, a condition sufficient to insure that previous inequality involving $\zeta$.
}{
Included in~\cite{CompositionHoppingTR}.
}
\end{proof}

% %%%%%%%%%%%%%%%%%%%%%%%%%%%%%%%%%%%%%%%%%%%%%%%%%%%%%%%%%%%%%%%%%%%%%%%%%%%%%%%%

% \begin{figure}[t]
% \minipage[b]{\columnwidth}\footnotesize
% \centering
% \def\svgscale{0.6}\import{../fig/}{model_full.pdf_tex} 
% \endminipage\hfill
% \caption{A hybrid 2DOF inertial reorientation template with two segments pinned at the CoM and no gravity. {\bf Left:} the net angular momentum of the system is constant. {\bf Right:} the system can correct the net angular momentum using reaction torques on the main body segment.}
% \label{fig:model_full}
% \end{figure}

%%%%%%%%%%%%%%%%%%%%%%%%%%%%%%%%%%%%%%%%%%%%%%%%%%%%%%%%%%%%%%%%%%%%%%%%%%%%%%%%

\iftoggle{techrep}{
\section{Physical System: Tailed, Compliant-legged Biped}
}{
\section{Physical System: Tailed Monoped}
}
\label{sec:TailedMonoped}

Our target physical platform is a tailed bipedal robot that we have built\iftoggle{techrep}{}{~\cite{CompositionHoppingTR}}, which (when planarized) we model as shown in the center of Fig.~\ref{fig:Tree}.
We were able to formally show template-anchor relations going from 1DOF to 2DOF templates (Propositions \ref{prop:SLIPComposition} and \ref{prop:HIRStability}), because of the availability of simple models (\S\ref{sec:MBHop}), or trivial dynamics (\S\ref{sec:HIR}). However, as we proceed up the desired hierarchy (Fig.\ \ref{fig:Tree}), there are no easily accessible tools that let us directly analyze the effects of coupling in the return map. In this section, we only show (Proposition \ref{prop:TMComposition}) that under a highly restrictive assumption \ref{def:RobotAssumptions} (that essentially makes the tail sweep negligible), the closed-loop tailed monoped return map $\hyb{F}{tm}$ has an invariant submanifold where it is equal to $\hyb{F}{s} \times \hyb{F}{a}$, but we also leave as conjecture that this invariant submanifold is attracting.

%%%%%%%%%%%%%%%%%%%%%%%%%%%%%%%%%%%%%%%%%%%%%%%%%%%%%

\iftoggle{techrep}{

The first two subsections of this section discuss (in an informal manner) the design process of the robot platform we have designed, built, and implemented the tail-energized hopping behavior on.

\subsection{Jerboa: Design and Construction}
\label{sec:jerboa}

\begin{figure}[t]
\minipage[m]{0.49\textwidth}
\centering
\includegraphics[height=5cm]{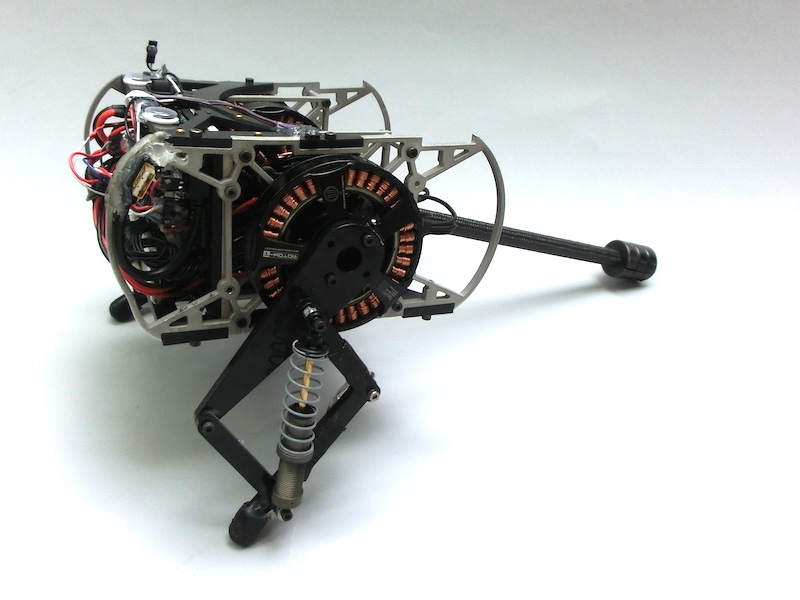}
\endminipage\hfill
\minipage[m]{0.49\textwidth}
\centering
\def\svgscale{0.75}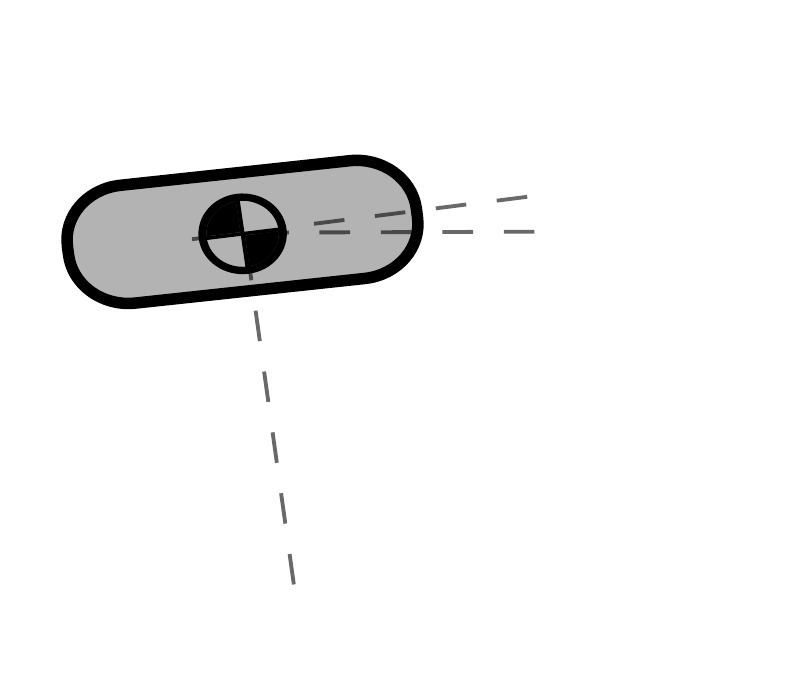
\endminipage\hfill
\caption{The Jerboa is a 2 Kg robot with hip-actuated legs and a 2DOF tail, pictured on the left as it appeared in the experiments of Section~\ref{sec:Experiments}. On the right is our model for the planarized 4DOF system for comparison.}
\label{fig:robot}
\end{figure}

The Jerboa was designed with the goal of being a dynamic, agile robot with an inertial appendage. 
We defer an in-depth discussion of morphological constraints and tradeoffs to future work, but present the following basic design decisions here:
\begin{enumerate}[i)]
\item With an eye on power density constraints\footnote{Adding actuated DOFs parasitically increases mass, but it is not a direct consequence that a proportionate amount of usable power will be added to the robot body by the extra DOFs.}, the robot is underactuated. There are 12 spatial DOFs (6 for the body, 2 for each revolute leg, 2 for the tail) and 4 actuators. When planarized with a boom and virtual constrains on the appendages (as we have done in this paper), there are 6 planar DOFs: 3 for the body, 2 for the single leg and 1 for the pitching tail. Raibert showed that an underactuated robot can be dynamically stable~\cite{raibert_legged_1986}, and in order to have the best performance, we limited the number of actuators on the robot to the minimum number that we believe is required to achieve a wide variety\footnote{We have some preliminary empirical evidence that the Jerboa can quasistatically and dynamically balance, in order to sit, stand, walk, hop, run, turn, leap, etc. Careful investigation of each of these behaviors is planned for future work.} of behaviors.
\item The body has low inertia (due to the mass of the motors being concentrated near the CoM, and the appendages being light), and the actuators are configured such that they can impart correspondingly large accelerations to the body (with an eye towards ``agility''). Future work is planned to reconcile our inclination with emerging definitions of specific agility~\cite{duperret_specific_agility_2014}, but intuitively it seems as if ``integrated magnitude of body acceleration'' is a reasonable metric to aim for.
\item The hips are actuated, but the leg extension is completely passive. This particular form of underactuated leg has been demonstrated to have great versatility in RHex~\cite{saranli_rhex:_2001}, for steady-state running as well as transitional maneuvers~\cite{johnson_toward_2013}.
\item The robot contains an inertial appendage which is endowed with the same amount of power as the hips. Recent research in biomechanics~\cite{libby_tail-assisted_2012} and robotics~\cite{johnson_tail_2012} has demonstrated the utility of tails as inertial ``self-righting'' devices, and on the Jerboa we promote it to a primary source of locomotory energy and control.
\end{enumerate}

In the remainder of this section, we outline the electromechanical aspects of the construction of the robot. A summary of important mechanical measurements is provided in table~\ref{tab:ParameterValues}. 

\begin{table*}[t]
  \caption{Parameter values}
  \label{tab:ParameterValues}
  \begin{center}
  \begin{tabular}{@{}llll@{}}
    \toprule
    Mass (with battery) &
    2.419 Kg &
    Dimensions (without tail) &
    0.21 m (L) $\times$ 0.23 m (W) $\times$ 0.1 m (H) \\
    Tail length &
    0.3 m &
    Tail mass &
    150 g \\
    Leg length &
    0.105 m &
    Leg motor stall torque & 
    3.5 N-m \\
    Peak power density & 
    376 W/Kg &
    Peak (vertical) force density & 
    46 N/Kg \\
    \bottomrule
  \end{tabular}
  \end{center}
\end{table*}

\subsubsection{2DOF tail}

\begin{figure}[t]
\centering
\includegraphics[width=\columnwidth]{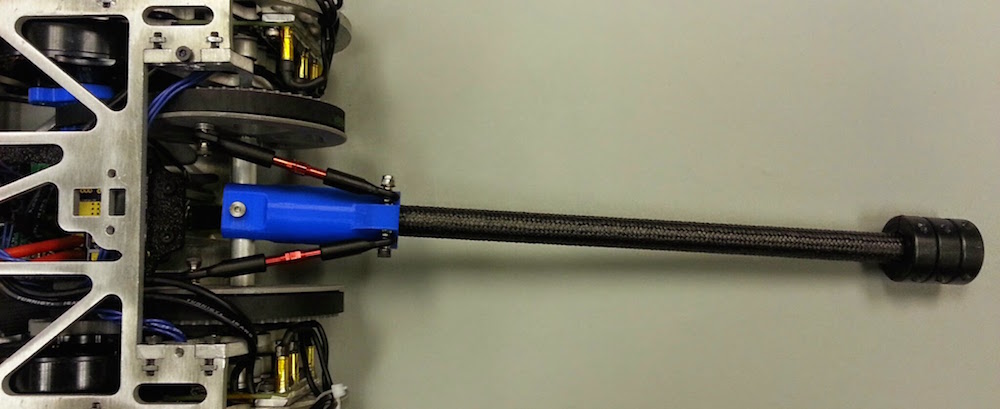}
\caption{The jerboa tail is a 2DOF spherical joint controlled using coaxial motors through a mechanical linkage. Though there are driven sprockets visible in this image, the version of the robot presented in this paper did not have this additional reduction stage.}
\label{fig:jerboataildesign}
\end{figure}

The tail appendage is configured as a 2DOF spherical joint with a point mass at the distal tip. The joint itself is constructed using a linkage (Fig.~\ref{fig:jerboataildesign}) such that identical motor displacements result in a pitching motion, and differential motor displacements result in a yawing motion. The forward kinematics map from motor angles $\mu_1, \mu_2 \in T^2$ to the tail pitch and yaw angles, $\phi_2, \phi_\mathrm{yaw} \in T^2$ has a simple form when restricted to zero yaw (i.e. $\mu_1 = \mu_2$),
\begin{align}
\phi_2(\mu_1, \mu_2)\restr{\phi_\mathrm{yaw}=0} = \mu_1 = \mu_2.
\end{align}
For the behavior under study in this paper, a virtual constraint ensures that $\phi_\mathrm{yaw} = 0$. We leave a full kinematic analysis of the 2DOF mechanism to future work.

\subsubsection{Prismatic-compliant revolute-actuated legs}

\begin{figure*}[t]
% \minipage[m]{0.12\textwidth}
% \centering
% \includegraphics[height=5cm]{psleg.jpg}
% \endminipage\hfill
% \minipage[m]{0.12\textwidth}
% \centering
% \includegraphics[height=5cm]{compressionleg.jpg}
% \endminipage\hfill
\minipage[m]{0.38\textwidth}
\centering
\includegraphics[height=5.8cm]{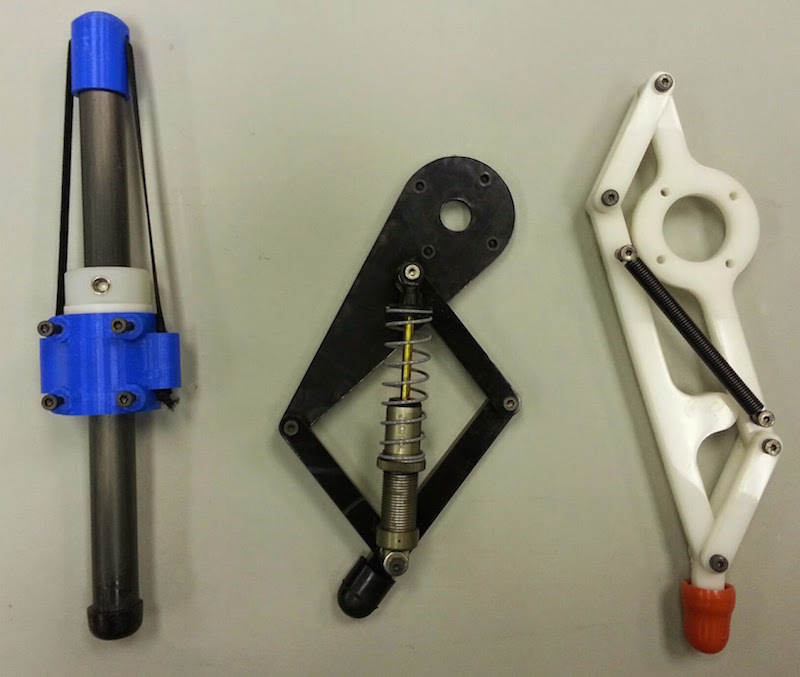}
\endminipage\hfill
\minipage[m]{0.6\textwidth}\small
\centering
\def\svgscale{0.92}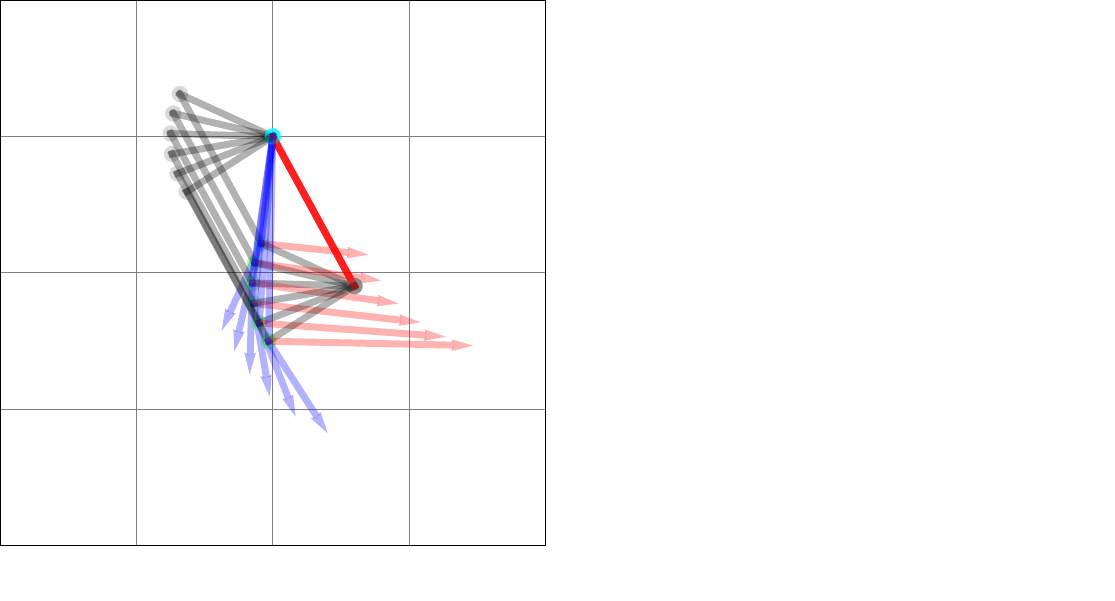 
\endminipage\hfill
% \minipage[m]{0.12\textwidth}
% \centering
% \includegraphics[height=5cm]{extensionleg.jpg}
% \endminipage\hfill
\caption{\textbf{Left:} Three leg designs considered for the Jerboa; the prismatic spring is ``ideal'' (in our model of \S\ref{sec:SLIP}, the spring force is dominantly axial, and the actuator force is predominantly tangential) but difficult to manufacture, and the four-bar designs only approximate the desired kinematics. \textbf{Right:} Configuration-dependent Jacobians of the compression and extension spring designs, where the displayed arrows map infinitesimal hip torques and spring extension forces to forces represented by red and blue (resp.) arrows at the toe. Out of these designs, the pictured version of the robot in Fig.~\ref{fig:robot} uses compression springs.}
\label{fig:jerboalegdesign}
\end{figure*}
\begin{figure*}[t]
\minipage[m]{0.3\textwidth}
\centering
\includegraphics[height=3.5cm]{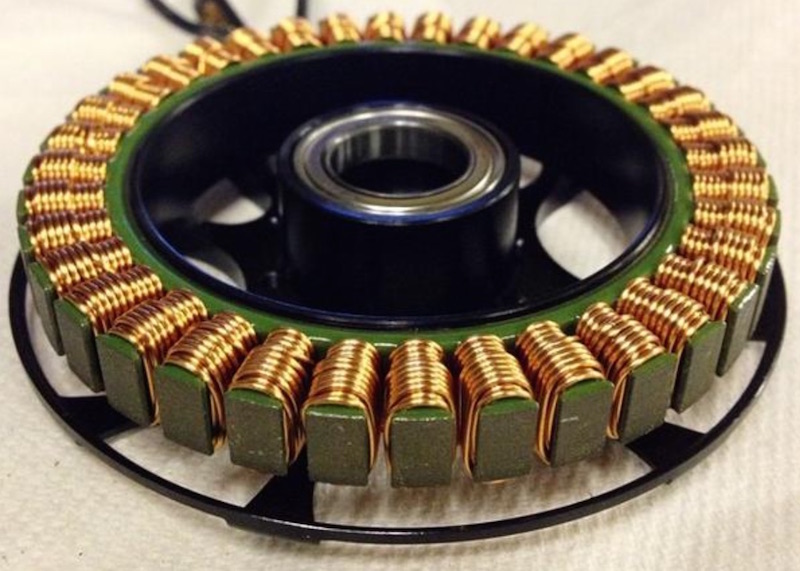} 
\endminipage\hfill
\minipage[m]{0.45\textwidth}
\centering
\footnotesize
\begin{tabular}{@{}lll@{}}
  \toprule
  & Maxon EC-45 & T-motor U8 \\
  \cmidrule{2-3}
  Mass (Kg) &
  0.11 &
  0.24 \\
  Gap radius (mm) &
  21.5 &
  45 \\
  $K_T$ (N-m/A) &
  0.033 &
  0.095 \\
  $K_\mathrm{TS}$ (N-m/Kg$^\circ$C) & 
  0.104 &
  0.5 \\
  \bottomrule
\end{tabular}
\normalsize
\endminipage\hfill
\minipage[m]{0.25\textwidth}\small
\centering
\def\svgscale{1}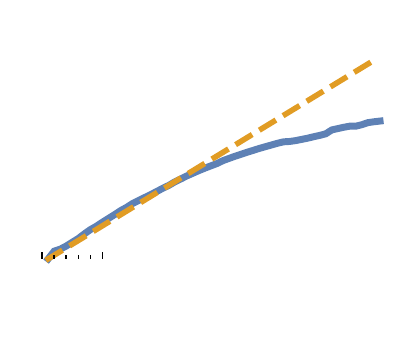 
\endminipage\hfill
\caption{\textbf{Left:} The selected actuator for the Jerboa is the \href{http://www.rctigermotor.com}{T-motor U8}, showing a thin profile and large gap radius---desireable properties for legged applications~\cite{seok2012actuator}. \textbf{Middle:} Motor properties relevant to selection for legged applications for the Jerboa motor, and the X-RHex~\cite{galloway_x-rhex:_2010} motor. \textbf{Right:} A torque-current plot for the U8 when coupled with our custom motor controllers of Fig.~\ref{fig:motorcontroller}, showing flux saturation at higher currents and a dashed line for the nominal torque (predicted by $K_T$).}
\label{fig:tmotor}
\end{figure*}

\begin{figure*}[t!]
\centering
\minipage[m]{0.48\textwidth}
\centering
\includegraphics[height=4cm]{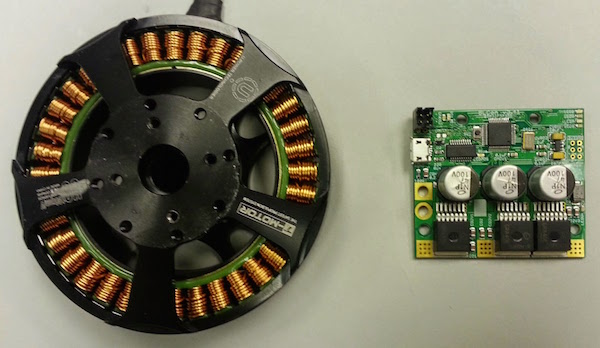} 
\endminipage\hfill
\minipage[m]{0.48\textwidth}
\centering
\includegraphics[height=4cm]{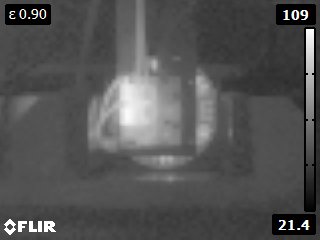} 
\endminipage\hfill
\caption{\textbf{Left:} The physical dimensions of our motor controller when compared to the motors they are driving. \textbf{Right:} Infrared image of of our actuation setup at stall, showing the controller reaching higher temperatures than the motor coils.}
\label{fig:motorcontroller}
\end{figure*}

Even though we adopt the underactuated hip-driven legs from RHex, the legs are chosen to have ``toes'' with point contacts instead of the rolling contact typical of RHex legs for the following reasons:
\begin{enumerate}[i)]
\item Our template plant for fore-aft speed control (\S\ref{sec:MBHop}) is an inverted pendulum with a point contact, and in particular, the toe-placement strategy for fore-aft speed control~\eqref{eq:RaibertController} is only (currently) well-understood for this leg structure.
\item The Raibert pitch controller~\cite{raibert_legged_1986}, which we use as part of our attitude control~\eqref{eq:HIRCL}, depends on a ``rigid'' connection between the hip and the toe. With a series-elastic element that may have torsional compliance (such as a C-leg), the ground reaction force would load up the leg spring, introducing the spring dynamics as a ``lag'' in our pitch control strategy.
\end{enumerate}

The left of Fig.~\ref{fig:jerboalegdesign} shows three leg designs that were considered for the Jerboa: 
\begin{inparaenum}[i)]
\item a prismatic mechanism with a nonlinear elastic element, 
\item a compression spring in a four-bar mechanism, and
\item an extension spring in a four-bar mechanism.
\end{inparaenum}
While the kinematic properties of the first design are the closest to our model (the spring force at the toe is purely radial, and the motor force at the toe is purely tangential), this design proved difficult to construct because of the linear bearing required. The kinematics of the ``approximate'' leg designs are pictorially depicted in Fig.~\ref{fig:jerboalegdesign}.

The experiments for this paper were all performed with the compression-spring legs. The compliant element is an off-the-shelf shock absorber for RC vehicles with lightweight construction, but considerable damping. We believe that the damping in legs was an important limiting factor in the energy of the hopping behavior demonstrated in \S\ref{sec:Experiments}.

\subsubsection{Actuators}

The power generated by electromechanical actuators tends to be at unusably high speeds for legged applications, however at the same time, higher gear reductions are undesirable due to a multitude of reasons~\cite{seok2012actuator}. To this end, we tune our actuator selection to maximize \emph{thermal specific torque}, $K_\mathrm{TS}$---the torque generated by the motor per unit mass per unit temperature rise. This modification to the torque density criterion of~\cite{seok2012actuator} allows us to incorporate the thermal implications of sustained motor activation\footnote{We are assuming a thermal dissipation model for the motor, but not accounting for temperature effects on magnetic flux density.}.
Fig.~\ref{fig:tmotor} contains a table comparing these metrics for the chosen actuator, a T-motor U8, and the one selected for X-RHex~\cite{galloway_x-rhex:_2010}, a Maxon EC-45.

Additionally, we developed custom motor controllers built around Infineon BTN8980 integrated half-bridges and an STM32F373 microcontroller that are
\begin{inparaenum}[(a)]
\item lightweight (20 g),
\item commutate using field-oriented control (FOC) at 25 KHz (adapted from~\cite{piccoli_cogging_2014}), 
\item deliver up to 55 A peak current and up to 40 V peak voltage, and
\item have built-in 12-bit rotor position sensing.
\end{inparaenum}
As a tradeoff for the high power-density of the driving electronics, they are limited by the heat dissipation ability of the half-bridges. Based on some crude testing, we have found that we can source approximately 10 A of steady-state current (thermally limited), corresponding to around 1 N-m of torque.
Fig.~\ref{fig:motorcontroller} compares the physical dimensions and thermal performance of the motor controllers to the motors we have chosen.
We note the following consequences of our selection of motor and driving electronics:
\begin{enumerate}[i)]
\item The high torque density of the chosen motors allows us to completely forgo any static gear-reduction on the Jerboa (although the 2DOF tail makes use of a linkage to transmit power to a spherical joint)---affording benefits of ``transparency'' and eliminating any transmission losses~\cite{seok2012actuator,kenneally_kinematic_2014}.

\item Power dissipation (to heat) in the motors is not a limiting factor in the robot's performance with the current driving electronics.

\item By eliminating the need for gearboxes and judicious chassis design, we have been able to reduce the ``robot framing cost'' to only 40\% of the mass of the robot. To put this in context, only 8\% of the mass of X-RHex is motors~\cite{galloway_x-rhex:_2010}.
\end{enumerate}

% \subsection{Design for Decoupled Control}

Lastly, we highlight some of the design aspects of the Jerboa that are particularly relevant to the subject of this report (tail-energized hopping via parallel playback of decoupled controllers):

\begin{assumption}[Design for decoupled control]
\label{def:RobotDesign}
The design of the Jerboa specifically ensures
\begin{inparaenum}[(i)]
\item leg/tail axes of rotation are coincident at the ``hip,''
\item the tail mass is small, i.e. $m_t \ll m_b$, and
\item center of mass (approximately configuration-independent by the previous assumption) coincides with the hip.
\end{inparaenum}
\end{assumption}

We point out here that these design decisions are less strict than the ones required for our present analysis (assumption~\ref{def:RobotAssumptions}). We believe that the stringency of assumption~\ref{def:RobotAssumptions} is not \emph{necessary}, and provide some empirical evidence to this effect in Section~\ref{sec:Experiments}.

% End toggle
}{}

%%%%%%%%%%%%%%%%%%%%%%%%%%%%%%%%%%%%%%%%%%%%%%%%%%%%%

\subsection{Modeling for Planar Hopping}

Raibert's planar hopper~\cite{raibert_legged_1986} empirically demonstrated stable hopping using a rigid body with a springy leg, and in this paper we pursue the same idea, but instantiate vertical hopping by coupling the 1-DOF leg-spring excitation controller (physically acting through the tail). In flight, the tail actuator grants us a new affordance that we only\iftoggle{techrep}{}{\footnote{We avoid a detailed discussion here, but a revolute tail avoids the morphological specialization of a dedicated prismatic actuator and can be repurposed for other uses such as static standing, reorienting the body in free fall~\cite{johnson_tail_2012}, directing reaction forces through ground contact for leaping when used as another ``leg''~\cite{johnson_toward_2013}, etc.}} use here to regulate the added ``shape'' DOF.
Our physical model is shown in Fig.\ \iftoggle{techrep}{\ref{fig:robot}}{\ref{fig:Tree} (center)}. 
The system has a single massless leg with joints $\theta = (\theta_1, \theta_2) \in S^1\times \bbR_+$, a rigid body $(x, z, \phi_1) \in \setwo$, and a point-mass tail with revolute DOF $\phi_2$, such that the full configuration is $\smq := (\theta_1, \theta_2, x, z, \phi_1, \phi_2) \in \sQ$. 
We make the following design-time assumptions:
\begin{assumption}
\label{def:RobotAssumptions}
\begin{inparaenum}[(i)]
\item Leg/tail axes of rotation are coincident at the ``hip,'' \label{CoMAssumption}
\item tail mass is small, i.e. $m_t \ll m_b$, \label{LightTail}
\item center of mass (configuration-independent by the previous assumption) coincides with the hip, and
\item body, tail have high inertia, i.e. $i_b, i_t \rightarrow \infty$.\footnote{Even though the dynamic task here is quite different from free-fall, in the language of \cite{johnson_tail_2012} this is saying that the tail should be light but \emph{effective}.}\label{BodyTailInertia}
\end{inparaenum}
\end{assumption}

\iftoggle{techrep}{
\subsection{Equations of Motion}

Using the self-manipulation \cite{johnson_legged_2013} formulation of hybrid dynamics, the inertia tensor is 
\begin{align}
\smM = \mat{0 & \\ & \smM_b}, \text{ where } \smM_b := \mat{\smM_1 & \smM_o^T\\\smM_o & \smM_2}.
\label{eq:massMatrices}
\end{align}
Note that $\smM_1 = (m_b + m_t)I$ and $\smM_2 = \mat{i_b+i_t & i_t\\i_t & i_t}$ are constant, and $\smM_o$ contains the critical cross-compartment interaction, by way of which we can use our tail actuator (formally acting on an attitude DOF, $\phi_2$) for energizing the shank DOF, $\theta_2$.

Let the forward kinematics of the leg be $\smg : \theta \mapsto \bbR^2$.
The constraint in the stance contact mode is
\begin{align}
\label{eq:aLeg}
\sma_1(\smq) = \mat{x\\z} - \R(\phi_1) \smg(\theta),
\end{align}
such that $\smA_1(\smq) = \mat{\R \D \smg & I & \J \R \smg & 0}$. In flight mode, $\sma_2(q) \equiv 0$. As in \cite{johnson_legged_2013}, the dynamics can be expressed as
\begin{align}
\mat{\smM & \smA_i^T\\\smA_i & 0} \mat{\ddot \smq\\\lambda} = \mat{\smU - \smN\\0} - \mat{\smC\\\dot \smA_i }\dot \smq.
\end{align}

Define the linear coordinate change $\smh : \sY = \sS \times \sA \rightarrow \sQ$, and $\smH := \D \smh$ such that 
\begin{align}
\label{eq:SMCoordChange}
\smh^{-1} : \smq \mapsto \mat{(\theta_1+\phi_1, \theta_2, x, z)^T \\ \smM_2 \mat{\phi_1 \\ \phi_2} },
\end{align}
and observe that $\smh^{-1}(\smq) = (s,a)$ is reminiscent of SLIP (\S\ref{sec:SLIP}) and attitude (\S\ref{sec:HIR}) coordinates. Define
\begin{align}
\pi_s := \mat{I_4 & 0}\smh^{-1},~~~~\pi_a := \mat{0 & I_2}\smh^{-1}
\end{align}

The equations of motion are generated in the new coordinates,
\begin{align}
\ddot \smy &= \smH^{-1} \smM^\dagger (\smU - \smN) - \smH^{-1} (\smM^\dagger \smC + \smA^{\dagger T} \dot \smA ) \smH \dot \smy.
\end{align}

In stance,
\begin{align}
\mat{\ddot s_1 \\ \ddot s_2} &= \mat{ \frac{\tau_h}{m_b \theta_2^2} - \frac{2 \dot\theta_2 \dot\theta_s}{\theta_2} \\ \frac{k_s (\rho_l - \theta_2)}{m_b} + \theta_2 \dot\theta_s^2} + \tfrac{\tau_t}{\rho_t m_b} \mat{\sin\xi/\theta_2 \\ -\cos\xi}, \label{eq:SLIPInStance}\\
\ddot a &= \mat{ -\tau_h \\ \tau_t } \label{eq:AttInStance},
\end{align}
where $\xi := \theta_1 - \phi_2$ (the tail-leg angle), and the right summand in (\ref{eq:SLIPInStance}) is quite clearly the disturbance caused due to the added attitude degrees of freedom.

With the same choice of $\smH$, we can similarly recover weakly decoupled flight dynamics:
\begin{align}
\mat{\ddot x \\ \ddot z} &= \mat{0 \\ -g} + \tfrac{\tau_t}{\rho_t m_b} \mat{\sin(\phi_1+\phi_2) \\ -\cos(\phi_1+\phi_2)}, \label{eq:SLIPInFlight}\\
\ddot a &= \mat{ 0 \\ \tau_t } \label{eq:AttInFlight}.
\end{align}
}
{
We derive the equations of motion in~\cite{CompositionHoppingTR}.
}

%%%%%%%%%%%%%%%%%%%%%%%%%%%%%%%%%%%%%%%%%%%%%%%%%%%%%

\subsection{``Physical'' Decoupling and Anchoring}
\label{sec:PhysicalDecoupling}

With the highly restrictive assumption~\ref{def:RobotAssumptions} (allowing for infinite tail inertia), the tail motion is essentially negligible. Under these conditions, we show the emergence of the beginnings of a classical anchoring relation~\cite{full_templates_1999}, via a natural (weak) decoupling of the 6DOF dynamics into ``point-mass'' and attitude compartments. A more general analysis that is more physically relevant is forthcoming in future work.

\begin{prop}[Flow-invariant submanifold]
\label{prop:TMInvariantSubmanifold}
Under assumption \ref{def:RobotAssumptions}, in each hybrid mode,
\begin{inparaenum}[(i)]
\item the submanifold $\sU = \{ T\smq \in T\sQ : T\phi_1=T\phi_2=0\}$ is invariant under the action of the flow generated by $\hyb{f}{tm}_i$, and
\item in each hybrid mode, the closed-loop flow restricted to $\sU$, $\dot{T \smq} = \hyb{f}{tm}_i (T\smq\restr{\sU})$ is a cross-product of the template vector fields,
\begin{align}
\hyb{f}{tm}_i = \hyb{f}{s}_i\circ\pi_s \times \hyb{f}{a}_i\circ\pi_a,
\end{align}
where $\pi_s$ and $\pi_a$ represent projections to the SLIP and attitude components of $\smq$ respectively.
\end{inparaenum}
\end{prop}

\begin{proof}
\iftoggle{techrep}{
Applying assumption \ref{def:RobotAssumptions}.\ref{LightTail} to the equations of motion, the plant dynamics $\hyb{p}{tm}(T\smq,(\tau_h,\tau_t))$ are
\begin{align}
\ddot\theta\restr{\mathrm{stance}} &= \mat{ {\color{BrickRed} \frac{\tau_h}{m_b \theta_2^2} } - \frac{2 \dot\theta_2 \dot\theta_s}{\theta_2} \\ \frac{k_s (\rho_l - \theta_2)}{m_b} + \theta_2 \dot\theta_s^2 } + {\color{BrickRed} \tfrac{\tau_t}{\rho_t m_b} \mat{\sin\xi/\theta_2 \\ -\cos\xi} }, \nonumber\\
\ddot a\restr{\mathrm{stance}} &= \mat{ -\tau_h \\ \tau_t }, \nonumber\\
\mat{\ddot x \\ \ddot z}\restr{\mathrm{flight}} &= \mat{0 \\ -g} + {\color{BrickRed} \tfrac{\tau_t}{\rho_t m_b} \mat{\sin(\phi_1+\phi_2) \\ -\cos(\phi_1+\phi_2)} },\nonumber\\
\ddot a\restr{\mathrm{flight}} &= \mat{ 0 \\ \tau_t },\label{eq:BodyOL}
\end{align}

We can check that we have available affordances through our two actuators to assign (scaled versions of) our template controllers in Table \ref{tab:controllers},
\begin{inparaenum}[(i)]
\item $\tau_h\restr{\mathrm{stance}} = -\hyb{g}{p}_1(a_1,\dot a_1)$ to control $a_1$, and $\tau_h\restr{\mathrm{flight}} = \hyb{g}{fa}_2(\dot x)$ to control $\dot x$, and
\item $\tau_t\restr{\mathrm{flight}} = \hyb{g}{sh}_2(a_2,\dot a_2)$ to control $a_2$, and $\tau_t\restr{\mathrm{stance}} = -\rho_t \theta_2 m_b \cdot \hyb{g}{v}_1(\dot z)$ to control hopping height\footnote{We observe that by assumption \ref{def:PendularAssumptions}.\ref{ConstantAngVel}, $\theta_2 \approx \rho_l$ is roughly constant, so the scaling need not be configuration dependent.}.
\end{inparaenum}

Under assumptions \ref{def:PendularAssumptions}.\ref{StanceSmallAngle} and \ref{def:RobotAssumptions}.\ref{BodyTailInertia}, we show that the {\color{BrickRed} highlighted terms} in (\ref{eq:BodyOL}) vanish inside $\sU$:
\begin{enumerate}[i)]
\item $\smM_2 \rightarrow \infty$, so in the dynamics equations $\ddot a = 0$. Restricted to $\sU$, $a\equiv 0$. This proves part (i) of the claim.
\item From $\ddot a\equiv 0$ and (\ref{eq:HIRCL}), $\tau_h \restr{\mathrm{stance}}=\tau_t \restr{\mathrm{flight}}=0$.
\item Since $\phi_2 = 0$, $\xi = -\phi_1 \approx 0$ (from assumption \ref{def:PendularAssumptions}.\ref{StanceSmallAngle}).
\end{enumerate}

By comparing the thus-restricted plant dynamics (\ref{eq:BodyOL}) to (\ref{eq:SLIPStance}), (\ref{eq:SLIPFlight}) and (\ref{eq:HIRbody}), we obtain part (ii) of the result.
}{
Included in~\cite{CompositionHoppingTR}.
}
\end{proof}
Additionally, the invariant submanifold in the flow leads to an invariant submanifold in the hybrid execution:
% We leave the missing piece as a conjecture, and present experimental evidence that it is indeed true in \S\ref{sec:Experiments}:
% \begin{conjecture}(Attracting invariant submanifold)
% \label{conj:TMAttracting}
% The set $\sU$ is attracting under the action of $\hyb{f}{tm}$.
% \end{conjecture}

%%%%%%%%%%%%%%%%%%%%%%%%%%%%%%%%%%%%%%%%%%%%%%%%%%%%%

% \subsection{Parallel Composition of Hybrid Templates}
% \label{sec:Composition6DOF}

% Building on Proposition \ref{prop:TMInvariantSubmanifold} and Conjecture \ref{conj:TMAttracting}, we can now make the following statement:
\begin{prop}[Return map-invariant submanifold]
\label{prop:TMComposition}
The set $\sU$ is invariant under the return map $\hyb{F}{tm}(T\smq\restr{\sU})$, and restricted to $\sU$, $\hyb{F}{tm} = \hyb{F}{s}\circ\pi_s\times\hyb{F}{a}\circ\pi_a$.
%anchors the template product, $\hyb{F}{s}\times\hyb{F}{a}$, and, hence anchors the sub-template product $\hyb{F}{v}\times\hyb{F}{fa} \times \hyb{F}{p}\times\hyb{F}{sh}$.
\end{prop}

\begin{proof}
\iftoggle{techrep}{
We first define the return map $\hyb{F}{tm}$ by instantiating a ``cross-product'' hybrid system $(\hyb{\sD}{tm}, \hyb{f}{tm}, \hyb{r}{tm})$ as
\begin{inparaenum}[(a)]
\item $\hyb{\sD}{tm} := \hyb{\sD}{s} \times \hyb{\widetilde\sD}{a}$,
\item $\hyb{r}{tm} := \hyb{r}{s} \times \hyb{\widetilde r}{a}$, and
\item $\hyb{f}{tm}$ as defined in Proposition \ref{prop:TMInvariantSubmanifold},
\end{inparaenum}
where $\hyb{\widetilde\sD}{a}_i := T S^2 \times S^1$ for each $i$ (ensuring $\partial \hyb{\widetilde\sD}{a} = \emptyset$) and $\hyb{\widetilde r}{a}_i : \hyb{\widetilde\sD}{a}_i \rightarrow \hyb{\widetilde\sD}{a}_{i+1}$ is defined
\begin{align}
\widetilde{\hyb{r}{a}_i} : \mat{Ta\\\psi_a} \mapsto \mat{Ta\\i \pi \bmod{2\pi}}.
\end{align}
With these modifications, the $\psi_a$ dynamics (\ref{eq:HIR}) are ignored, and the clock of the HIR subsystem is being driven by the SLIP subsystem\footnote{This coupling interaction importantly invalidates the $\omega_a$-dependent bound on $k$ (\ref{eq:HIRkbound}). Our solution is to scale the input such that $k$ is high enough for the shortest feasible transition time in vertical hopping.}. This ensures that the conditions of Proposition \ref{prop:HIRStability} still hold, i.e. $\pi_a\circ\hyb{F}{tm} = \hyb{F}{a}\circ\pi_a$.

Additionally, the decoupled nature of $\hyb{f}{tm}\restr{\sU}$ (Proposition \ref{prop:TMInvariantSubmanifold}) allows us to conclude that $\pi_s\circ\hyb{F}{tm} = \hyb{F}{s}\circ\pi_s$, so that \[
\hyb{F}{tm} = \pi_s\circ\hyb{F}{tm} \times \pi_a \circ \hyb{F}{tm} = \hyb{F}{s} \circ \pi_s \times \hyb{F}{a}\circ\pi_a,
\]
which concludes the proof.
}{
Included in~\cite{CompositionHoppingTR}.
}
\end{proof}
We leave to future work a proof that $\sU$ is attracting, which is a requirement for demonstration of anchoring~\cite{full_templates_1999}.

%%%%%%%%%%%%%%%%%%%%%%%%%%%%%%%%%%%%%%%%%%%%%%%%%%%%%

% \subsection{Limiting Behavior}
% \label{sec:LimitingBehavior}

% The 
% Clock synchronization: 

% \begin{itemize}
% \item Physical system is driving the clock! All phase differences are deadbeat.
% \item \fxerror{Want to say: in general, composition should go hand-in-hand with model reduction \cite{burden_model_2013}, so that we make sure that by joining two $S^1$ limit cycles we end up with a $T^2$. Don't know how to say this yet.}
% \end{itemize}

%%%%%%%%%%%%%%%%%%%%%%%%%%%%%%%%%%%%%%%%%%%%%%%%%%%%%%%%%%%%%%%%%%%%%%%%%%%%%%%%

\section{Experimental Results}
\label{sec:Experiments}
 
%%%%%%%%%%%%%%%%%%%%%%%%%%%%%%%%%%%%%%%%%%%%%%%%%%%%%%%%%%%%%%%%%%%%%%%%%%%%%%%%

\iftoggle{techrep}{
% TECH REPORT
In this section we present empirical data obtained from the Jerboa (\S\ref{sec:jerboa}). In the first three subsections, we present data from a few ``nodes'' of our composition tree (Fig.~\ref{fig:Tree}). Finally, a crucial examination of our idea of composition of templates, when implemented on the Jerboa, is presented in \S\ref{sec:ExperimentsComposition}.

\subsection{Effect of Varying Tail Mass on Vertical Hopping}

\begin{figure*}[t]
\centering
\def\svgscale{0.8}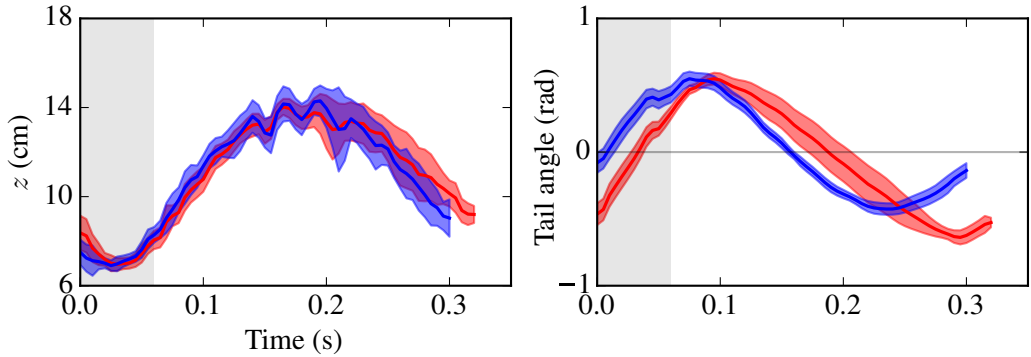 
\caption{Two datasets corresponding to different tail masses: The blue traces use the $m_t = 150$ g (as in Table~\ref{tab:ParameterValues}), but the red traces use $m_t = 100$ g. Note that the tail displacement is larger for the lighter tail mass, although vertical behavior is largely unaffected.}
\label{fig:tailmass}
\end{figure*}

The first empirical result we present corresponds to the top left leaf of Fig.~\ref{fig:Tree}---empirical vertical hopping. In order to facilitate the analysis in this paper, in assumption~\ref{def:RobotAssumptions} we stipulated an ideally effective~\cite{johnson_tail_2012} tail, with negligible mass and infinite inertia. We connected the robot (Fig.~\ref{fig:robot}) to a boom and constrained the body pitch as well as the fore-aft DOF. By varying the tail mass (with a fixed tail length given in Table~\ref{tab:ParameterValues}), we obtained two vertical hopping datasets plotted in Fig.~\ref{fig:tailmass}.

We observe the following:
\begin{enumerate}[i)]
\item Increasing tail mass results in smaller tail displacements. Taken to the limit, this sheds some light on assumption~\ref{def:RobotAssumptions}: a large tail mass would indeed render the tail motion negligible.

\item The hopping height remains relatively unchanged in spite of this physical variation. From~\eqref{eq:SLIPInStance}, the force acting on the leg-spring depends only on the (feedforward) tail torque, $\tau_t$ (as in Table~\ref{tab:controllers}).
\end{enumerate}

Consequently, we see that the tail mass is a \emph{tunable design parameter} that allows us to trade off the conditions of assumption~\ref{def:RobotAssumptions} (negligible mass versus large inertia---both affecting coupling interactions) without affecting the vertical behavior.

\subsection{Empirical Validation of Attitude-Decoupling Change of Coordinates}
\label{sed:expPitchShape}

\begin{figure*}[t]
\centering
\def\svgscale{0.8}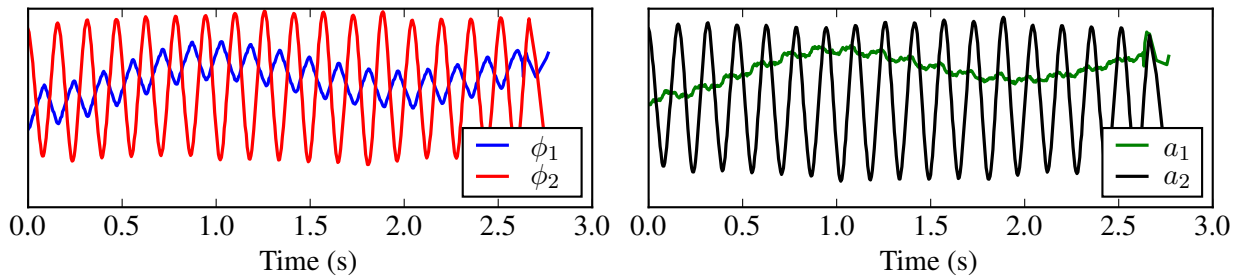 
\caption{Testing our decoupling change of coordinates from the physical body pitch, tail angle coordinates $(\phi_1,\phi_2)$ to our chosen attitude coordinates $(a_1, a_2)$ by suspending the robot about its CoM (see \S\ref{sed:expPitchShape}).}
\label{fig:pitchshape}
\end{figure*}

An important foundation of our attitude control strategy is the decoupling of the two attitude DOFs (\S\ref{sec:HIR}), such that $a_1$ is controlled in stance, and $a_2$ is controlled in flight~\eqref{eq:HIRbody}. However, the body pitch and tail angle are clearly coupled in flight\footnote{Since the tail actuator is attached between the body and the tail, tail torques are felt by the body.}. To resolve this, as shown in~\eqref{eq:SMCoordChange}, we use $\smM_2^{-1}$ as a decoupling change of coordinates.

In terms of implementation this strategy requires the estimation of a single scalar parameter that defines $\smM_2$ up to scale (see the text just after~\eqref{eq:massMatrices}). To test our the change of coordinates empirically, we suspended the robot about the CoM and applied a feedforward sinusoidal $\tau_t$ signal. The resulting traces for the physical attitude coordinates are shown in Fig.~\ref{fig:pitchshape}. 

Recall from~\eqref{eq:AttInFlight} that in flight, $\ddot a_1 = 0$. In practice, we observe from the right of Fig.~\ref{fig:pitchshape} that there are small $a_1$-variations are at a much slower time scale than $a_2$-variations. The reason that $\ddot a_1$ is not zero is that we were unable to suspend the robot at precisely the CoM, and so gravity exerts a net moment on the body---appearing as a slow $a_1$-oscillation. Other than this minor deviation of our physical platform from assumption~\ref{def:RobotAssumptions}, it appears as if the attitude-decoupling change of coordinates is indeed effective.

\subsection{Trading off Forward Speed and Hopping Height for ``Leaping''}
\label{sec:experimentsForeAft}

\begin{figure*}[t]
\centering
\includegraphics[width=15cm]{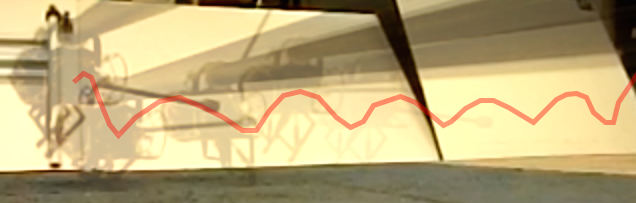}
\def\svgscale{0.8}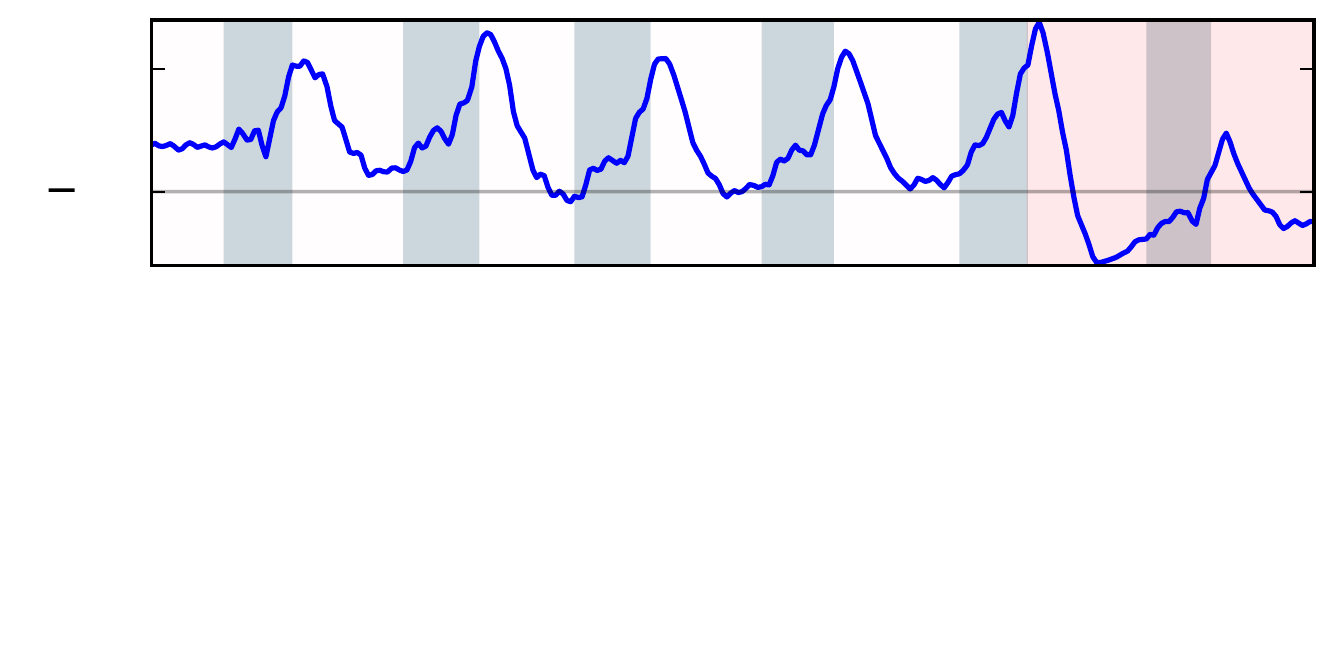 
\caption{\textbf{Top:} Snapshots of fore-aft hopping behavior in a trial where we test a ``leaping'' motion~\cite{raibert_legged_1986}---the robot stubs its toe at the last touchdown in order to gain a boost in vertical height at the expense of forward speed (see \S\ref{sec:experimentsForeAft}). The red line shows the CoM-trajectory of the robot. \textbf{Bottom:} Corresponding traces showing near-steady-state behavior in the fore-aft compartment (leg angle, $\theta_1$ and vertical height, $z$ are plotted) before the ``stubbing'' event (red overlay). The leg angle shows the ``neutral angle'' with a thin horizontal line, and in order to leap; note that a much larger (in magnitude from vertical) touchdown angle is chosen in order to leap. The leg height ($z$) plot shows the robot getting around 50\% larger apex height in the subsequent flight phase.}
\label{fig:foreaft}
\end{figure*}

The ``stepping'' fore-aft control using the touchdown angle as a control input~\eqref{eq:RaibertController} essentially allows us to trade off vertical and fore-aft energy---appearing as a pure rotation in~\eqref{eq:BHopIso}. Even though for steady-state behavior we choose the touchdown angle to stabilize forward speed, it also allows for transient behaviors such as a one-shot ``leaping'' motion (term coined by Raibert~\cite{raibert_legged_1986}). In particular, choosing a larger (in magnitude from vertical) touchdown angle than that dictated by~\eqref{eq:RaibertController} results in added vertical height and reduced fore-aft speed. 

The results of an empirical test of this one-shot leaping strategy are showing in Fig.~\ref{fig:foreaft}: we can indeed get a large increase in apex height using this strategy. This kind of ``asymmetry''~\cite{raibert_legged_1986} or deviation from steady-state may have important applications in behaviors that require rapid changes in the body energy, and we plan to explore more such behaviors in future work.

\subsection{Empirical Validation of Composition}
\label{sec:ExperimentsComposition}

}{
% ICRA paper
We perform the experiments on the Penn Jerboa: a new tailed bipedal robot platform (Fig.\ \ref{fig:robot}) with a pair of compliant hip-actuated legs (in parallel for sagittal plane behaviors), and a 2DOF revolute point-mass tail \cite{johnson_tail_2012} driven differentially by two motors through a five-bar mechanism (locked in the sagittal plane for the behaviors in this paper). We include a detailed design report
as well as additional experimental results including the effect of varying tail mass, and empirical validation of our pitch/shape decomposition of \S\ref{sec:HIR}
in~\cite{CompositionHoppingTR}.
}

% COMPOSITION FIGURE
\begin{figure*}[t]
% \minipage[m]{0.17\textwidth}
% \centering
% \includegraphics[width=3.2cm]{../fig/vidcap_stance_new_s}
% \endminipage\hfill
% \minipage[m]{0.17\textwidth}
% \centering
% \includegraphics[width=3.2cm]{../fig/vidcap_flight_new_s}
% \endminipage\hfill
\minipage[m]{\textwidth}
\centering
\iftoggle{techrep}{
\def\svgscale{1.0}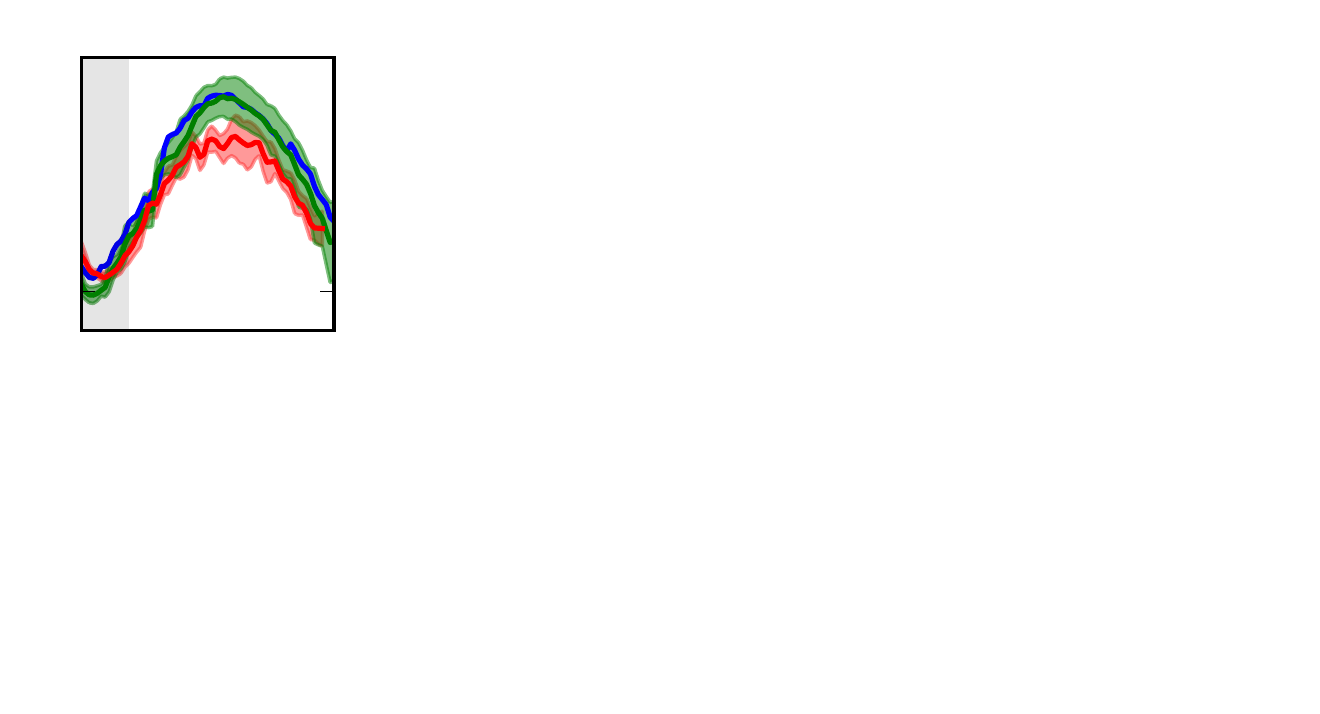 
}{
\footnotesize
\def\svgscale{0.63}\import{./}{compositionPlots.pdf_tex} 
}
\endminipage\hfill
\caption{A single stride (stance with shaded background followed by flight), where each column corresponds to some representative time series from each of the four 1DOF templates from \S\ref{sec:SLIP}-\ref{sec:HIR}, and the traces (mean and standard deviation) correspond to different ``bodies'' realized by variably constraining the robot---%black: simulation of the 1DOF plant corresponding to each templates, 
{\bf red:} tailed vertical hopper (i.e. $(\theta_1,x,\phi_1)$ locked), {\bf green:} tailed point-mass hopper (i.e. $\phi_1$ locked), {\bf blue:} tailed planar hopper (all free)---in which these templates are being anchored.}
\label{fig:compositionPlots}
\end{figure*}

By physically constraining some of the DOFs, we test our hierarchical composition (Fig.\ \ref{fig:Tree}) at as many ``nodes'' of the composition tree as possible. Note that it is infeasible to isolate the fore-aft or the closed-loop pitch correction templates in a physical setting. The results are summarized in Fig.\ \ref{fig:compositionPlots}. Five strides are averaged within each category, and aligned with ground truth knowledge of the touchdown event. We observe that
\begin{enumerate}[i)]
\item there is a vertical limit cycle that retains its rough profile and magnitude through three anchoring bodies,
\item the hip angle roughly satisfies $\ddot\theta_1=0$ in stance and the stance duration is roughly constant (corroborating assumptions \ref{def:PendularAssumptions}.\ref{ConstantAngVel}-\ref{ConstantStanceTime}, and our MBHop model (\ref{eq:MBHop}),
\item the shape coordinate is destabilized in stance and stabilized in flight, and the pitch-deflections are small in magnitude over the stride, and in agreement with (\ref{eq:HIR}).
\end{enumerate}

Qualitatively, the ``tailed point-mass hopper'' configuration attained stable forward hopping at controlled speeds upwards of 20 strides, only limited by space. The fully unlocked system has so far hopped for about 10 strides at multiple instances before failing due to accumulated error causing large deviations from the limit cycle. We believe the prime reason for this is that the CoM is significantly aft of the hip (violating assumption \ref{def:RobotAssumptions}.\ref{CoMAssumption}). We attempted to compensate for this effect with a counterbalance visible in Fig.\ \ref{fig:robot}, but an unacceptably large weight would have been required to completely correct the problem.

In the video attachment, we include clips of the robot hopping along a boom, with varying degrees of physical constraint corresponding to the ``bodies'' of Fig.~\ref{fig:compositionPlots} (annotated in the video). The controller implemented on the hardware is agnostic of the physical constraint, and takes the decoupled form of a cross-product of the rows of Table~\ref{tab:controllers}. 

%%%%%%%%%%%%%%%%%%%%%%%%%%%%%%%%%%%%%%%%%%%%%%%%%%%%%%%%%%%%%%%%%%%%%%%%%%%%%%%%

\section{Discussion and Conclusion}

Raibert's hopper \cite{raibert_legged_1986} made significant empirical advances in the field of  robotics, but to our knowledge, no
previous  account in the literature has provided any formal conditions under which such simple and decoupled control strategies will work.
In this paper, we apply simple decoupled controllers using similar ideas (including the exact same fore-aft (\ref{eq:RaibertController}) and pitch (\ref{eq:HIR}) controllers), but with a new vertical hopping scheme (\S\ref{sec:Radial}) and a new tail appendage to enable it. Moreover, we construct abstract models (that appear to, nevertheless, be representative of empirical data) that enable us to present analyses of stability for each of these subsystems, and make steps  towards a local proof of stability for the tailed hopper (a subject of future work by the authors). 

The first focus of future work is a complete analysis of stability of tail-energized hopping on the Jerboa, and development of formal tools for design and verification of parallel composition.
Second, our analysis in this paper is very specifically targetted to the tailed hopper (including the hand-designed hierarchy in Fig.\ \ref{fig:Tree}), but in future work we plan to generalize these ideas to other tasks as well as platforms. As explained in \S\ref{sec:Prelim}, we focus on closed-loop templates in this paper, but there is an accompanying interesting problem of assignment of actuator affordances to the control of specific compartments.
ar

%%%%%%%%%%%%%%%%%%%%%%%%%%%%%%%%%%%%%%%%%%%%%%%%%%%%%%%%%%%%%%%%%%%%%%%%%%%%%%%%

% asd
\iftoggle{techrep}{

\section*{Acknowledgments}

This work was supported in part by the ARL/GDRS RCTA project, Coop.\ Agreement \#W911NF-10–2−0016 and in part by NSF grant \#1028237. 
The authors would like to thank Gavin Kenneally for help with actuator selection and general design feedback on the Jerboa, and Shafag Idris for motor testing.

% \addtolength{\textheight}{-8.5cm}
}{}

%%%%%%%%%%%%%%%%%%%%%%%%%%%%%%%%%%%%%%%%%%%%%%%%%%%%%%%%%%%%%%%%%%%%%%%%%%%%%%%%
\bibliographystyle{unsrtnat}

% CHANGE BEFORE SUBMISSION TO BBL
% \bibliography{../../../drafts/biped_refs}
\bibliography{compositionTR}

\end{document}